\documentclass[11pt]{article}

\usepackage[final]{acl}

\usepackage{times}
\usepackage{latexsym}

\usepackage[T1]{fontenc}

\usepackage[utf8]{inputenc}

\usepackage{microtype}

\usepackage{inconsolata}

\usepackage{graphicx}

\IfFileExists{fontawesome5.sty}{%
  \usepackage{fontawesome5}%
}{%
  \providecommand{\faGithub}{}%
  \providecommand{\faGlobe}{}%
}

\usepackage{amsmath}
\usepackage{amsthm}
\usepackage{amssymb}
\usepackage{booktabs}
\usepackage{multirow}
\usepackage[breakable,listings]{tcolorbox}
\newtcolorbox{promptbox}[1]{
  colback=gray!10,
  colframe=gray!50,
  title=#1,
  halign title=center,
  breakable,
  fontupper=\scriptsize
}
\newtcblisting{dialoguebox}{
  listing only,
  colback=gray!10,
  colframe=gray!50,
  breakable,
  listing options={
    basicstyle=\scriptsize\ttfamily,
    breaklines=true,
    breakatwhitespace=true,
    columns=fullflexible,
    keepspaces=true,
    aboveskip=0pt,
    belowskip=0pt
  }
}

\newtheorem{definition}{Definition}
\newtheorem{proposition}{Proposition}

\newcommand{\tax}{collaboration tax}

\title{The Collaboration Tax: \\How Much LLM Multi-Agent Systems Pay to Coordinate}

\author{
    \textbf{Weixiang Sun}$^{1}$\;
    \textbf{Zehong Wang}$^{1,\dagger}$\;
    \textbf{Hong Huang}$^{2}$\;
    \textbf{Colby Nelson}$^{1}$\;
    \textbf{Yijun Ma}$^{1}$\;
    \textbf{Yanfang Ye}$^{1,\dagger}$ 
    \\
    \textsuperscript{1} University of Notre Dame\; 
    \textsuperscript{2} Simon Fraser University\; 
    \\
    $^\dagger$ Corresponding Authors
    \\
    \texttt{<wsun4,zwang43,yye7>@nd.edu}
    \\[2pt]
    \href{https://github.com/Weixiang-Sun/collaboration-tax}{\faGithub~Code}\quad
    \href{https://weixiang-sun.github.io/collaboration-tax/}{\faGlobe~Website}
}

\begin{document}
\maketitle

\begin{abstract}
Multi-agent systems built from large language models are deployed widely, yet how much performance is lost when two LLMs must coordinate rather than act alone remains unclear. We formulate the \emph{\tax} as the team-decentralisation loss of a two-player cooperative game with private information, with two propositions characterising its sign and its equivalence to a max-superadditivity violation. We operationalise this definition on $32$ solo-tractable tasks grouped by source of grounding friction and measure it on $11$ models from $7$ providers. The tax is structured along two no-exception axes: a category ordering across every model and a monotonic decrease with capability. The proximate mechanism is not a reasoning deficit but a four-stage conversational cascade in which agents make ungrounded claims, fail to query the partner, skip integrating both views, and accept the answer without re-derivation. The tax is mechanically predictable from conversation features and partly tractable: a prompt intervention targeting all four stages closes a substantial fraction of the gap, with the dominant bottleneck differing across categories. In heterogeneous pairs the tax is pulled toward the stronger partner rather than the additive midpoint, empirically realising the max-superadditivity violation predicted by our framework. Together these results recast collaboration in LLM systems as a measurable, predictable, and partly tractable cost.
\end{abstract}

\section{Introduction}
\label{sec:intro}

\begin{figure}[t]
    \centering
    \includegraphics[width=\linewidth]{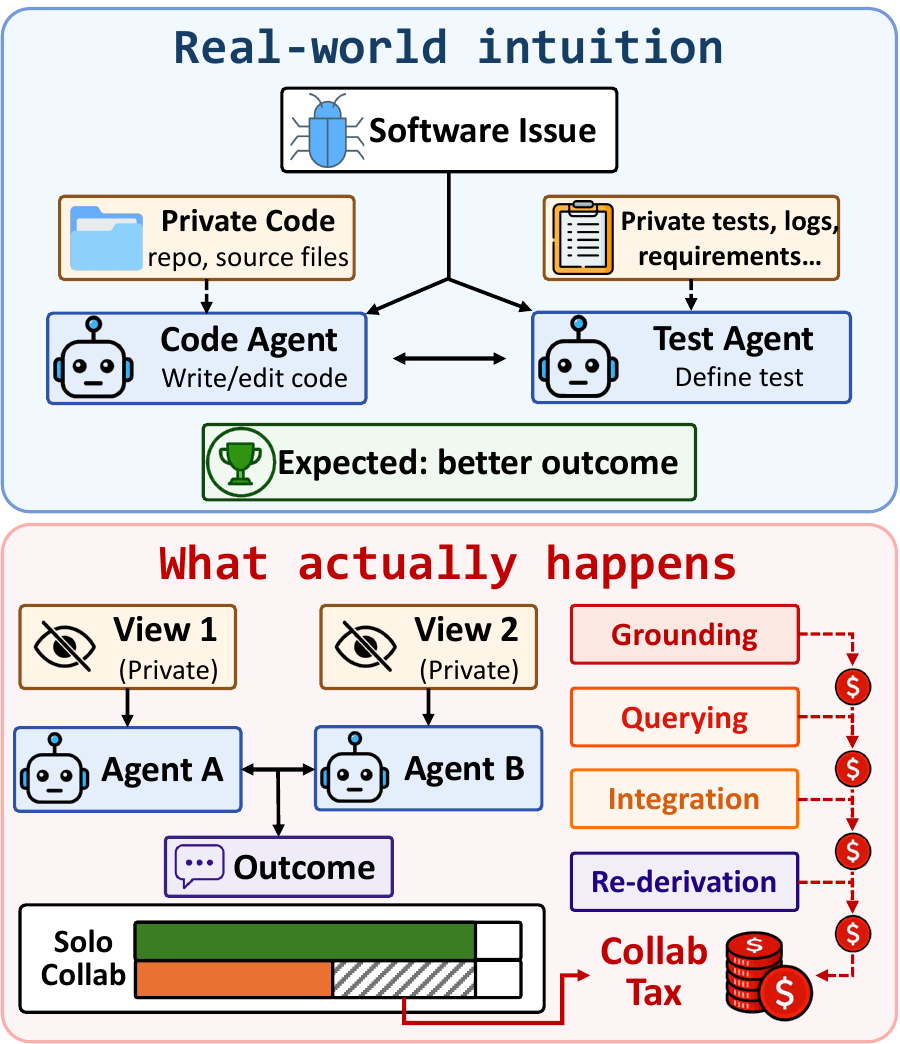}
    \caption{\textbf{The collaboration tax.} Two LLM agents with private views on a shared problem (top) are expected to outperform either alone; in practice, paired execution underperforms the solo upper bound by a \emph{collaboration tax} that decomposes into four conversational stages: grounding, querying, integration, and re-derivation (bottom).}
    \label{fig:teaser}
\end{figure}

Multi-agent systems built from large language models have become a fast-growing line of work \citep{guo2024large,ye2025llms4all} and now drive real applications across software engineering, planning, social simulation, and inference-time debate \citep{qian2024chatdev,hong2024metagpt,wu2024autogen,li2023camel,yuan2024mora,park2023generative,huang2024social,ma2026policy4ood,sun2026prescam,du2024improving,wang2026reasoning,ma2026proplay}. These deployments treat collaboration as a free primitive: assemble enough capable models, give them clear roles, and the team will outperform any single member. The premise is rarely tested directly. As the unit of deployment shifts from a single LLM to a collection of communicating ones, a basic empirical question shifts with it: when two or more LLMs must coordinate to solve a problem they could each handle alone, do they actually succeed together, and if not, why?

A small set of recent benchmarks has begun to probe this question directly \citep{davidson2025collaboration,eisenstein2026mt,sun2025collab,liu2024agentbench,yadav2026more,cemri2026multi}, but each measurement is local to one task domain or one pair configuration. The literature still lacks a structured, model- and task-agnostic account of coordination cost itself, with a known sign and a decomposition across task types and pair configurations (Section~\ref{sec:related} discusses each line in detail).

Three open questions structure our investigation. \textbf{First, is there a systematic coordination cost when LLMs are paired, and what governs its size?} The question asks whether the collaboration gap is a structural property of paired LLM execution or an artefact of one task type, and if structural, whether its size is driven primarily by model capability, by task structure, or by their interaction. \textbf{Second, when failure occurs, where in the conversation does it originate?} The breakdown could lie in reasoning (the model cannot compute the answer), in information access (the model lacks the relevant facts), in inter-agent communication (the facts exist but go unshared), or in verification (the receiver accepts a wrong answer); each candidate location implies a different intervention. \textbf{Third, can the cost be reduced without retraining?} This question separates structural conversational failures, which prompt-level changes can plausibly target, from intrinsic capability failures, which they cannot. We address all three.

To address these questions we formulate the \tax\ as a task-agnostic quantity anchored in the team-decentralisation framework of \citet{marschak1955elements} and the cooperative game theory of \citet{shapley1953value}, with two Propositions that characterise its sign and tie its positivity to a structural failure of the underlying cooperative game (Section~\ref{sec:tax}). We operationalise this definition on a $32$-task suite grouped into three families chosen to span three distinct sources of grounding friction in the sense of \citet{clark1991grounding,brennan1996conceptual,pickering2004toward}: Spatial (reference-frame alignment), Relational (lexical and category alignment), and CSP (indexing and ordinal alignment). Every task is solo-tractable by design, so the gap reflects coordination rather than problem-solving capacity. We evaluate $11$ models from $7$ providers in solo, homogeneous, and heterogeneous configurations.

We find that the \tax\ is not a single number but a structured, mechanistic phenomenon. The cost has two no-exception orderings, one over task categories and one over model capability, and is mechanically predictable from conversation-shape features alone, with capability setting the intercept of the gap and conversation shape its slope. The proximate failure mechanism is a \emph{four-stage cascade}, not a reasoning failure: from \emph{grounding}, where agents fabricate facts not stated by either view, to \emph{querying}, where they fail to ask the partner for the facts they do not have, to \emph{integration}, where they skip combining the two views before committing, and finally to \emph{re-derivation}, where the receiver accepts the answer without recomputing it. The mechanism extends to heterogeneous pairs: the cost is asymmetrically borne by the stronger partner, while the cascade dimensions continue to discriminate failure from success.

Building on this cascade insight, we design a prompt-level combined intervention that appends one stage-targeted clause for each of the four stages to the system prompt. The combined intervention closes a substantial fraction of the tax; per-stage leave-one-out ablations show that the dominant bottleneck differs across categories. The simplicity of the fix is exactly the point: failures attributed to reasoning or capability cannot be patched this cheaply, but failures of grounding, querying, integration, and re-derivation can.

Our contributions are as follows:
\begin{itemize}
  \item \textbf{A formal definition of the \tax} as the success-rate gap between full-information solo and split-view paired execution, with a cooperative-game interpretation that isolates coordination cost from problem-solving capacity.
  \item \textbf{A structurally categorised task suite} of $32$ solo-tractable tasks across Spatial, Relational, and CSP families, evaluated on $11$ models from $7$ families in solo, homogeneous, and slot-swapped heterogeneous configurations.
  \item \textbf{Mechanistic insights into the \tax}: a 2D landscape over capability and category, a four-stage conversational cascade (grounding, querying, integration, re-derivation) rather than a reasoning failure, and a regression decomposition in which capability sets the intercept and conversation shape sets the slope.
  \item \textbf{Stage-targeted prompt interventions that partially close the gap}, with a leave-one-out ablation showing that the dominant bottleneck layer differs across categories.
\end{itemize}

\section{Related Work}
\label{sec:related}

\paragraph{LLM-based agent collaboration.}
The closest prior work is \citet{davidson2025collaboration}, who introduce the \emph{collaboration gap} on a split-view maze and show that a strong-primer relay recovers most of it. Related benchmarks probe coordination from other angles: turn-count effects \citep{eisenstein2026mt}, Overcooked-style cooperation \citep{sun2025collab}, single-agent agentic tasks \citep{liu2024agentbench}, planner-with-rule-based-partner pairs \citep{zhang2024building}, zero-cost cooperation games \citep{yadav2026more}, and failure-mode taxonomies \citep{cemri2026multi}. A parallel engineering line orchestrates LLMs as cooperating agents \citep{wu2024autogen,hong2024metagpt,li2023camel,qian2024chatdev} and uses debate to improve reasoning \citep{du2024improving}; theory-of-mind probes \citep{kosinski2024evaluating,strachan2024testing} and cooperative reinforcement learning \citep{zhou2025sweet} examine related capacities without a structural account of where coordination fails. The classical theory of grounding \citep{clark1991grounding,clark1996using,levinson2003space} motivates the multiple equivalent representations our tasks exploit as natural friction. Each prior thread uses either a single task type or a homogeneous pair; our $32$-task suite across three structural families and a heterogeneous-pair matrix supplies the structured diagnostic these threads lack.

\paragraph{Cooperative games with asymmetric information.}
Our split-view setup inherits a longer tradition of two-player asymmetric-information cooperation. \citet{bard2020hanabi} formalise Hanabi as a benchmark for partial-observability cooperation, and recent neural agents \citep{v2025generalist} pursue zero-shot coordination across unseen partners. Overcooked-AI \citep{carroll2019utility} adds embodied real-time cooperation between learned policies and human surrogates, and textual referent games such as OneCommon \citep{udagawa2019natural} test grounding under privately-held continuous context. In dialogue, \citet{li2025grounded} annotate misunderstandings in MapTask transcripts to study how grounding breaks down in human dyads. Across all of these, the underlying problem is non-trivial even with full information, so coordination cost is entangled with problem-solving difficulty, and most use binary success metrics that wash out the continuous signal needed to distinguish near-misses from catastrophic failures. Our tasks are instead deliberately \emph{solo-trivial} and graded on a continuous $[0,1]$ scale, so the gap between solo and collaborative performance can be attributed to coordination rather than to a model's underlying ability to solve the task, and partial progress is preserved as a graded outcome.

\section{The Collaboration Tax}
\label{sec:tax}

To compare coordination cost across tasks, models, and pair configurations, and to know what the resulting quantity means when it is zero, positive, or transferred between dyads, we need a task-agnostic definition with a known sign and a structural interpretation rather than a number reported in isolation on a single task. \citet{davidson2025collaboration} make the cost explicit on a split-view maze by comparing one LLM given the full instance to two copies coordinating over complementary halves; we generalise this construction to any task admitting a union-necessary partition and a continuous deterministic grader, and anchor the resulting quantity as the team-decentralisation loss of \citet{marschak1955elements} for a two-player cooperative game with private information \citep{shapley1953value}. Two Propositions, stated and proved in Appendix~\ref{app:theory}, pin down its structure: \emph{information dominance} guarantees that the \tax\ is non-negative whenever the paired protocol's policy is dominated by the optimal centralised policy, explaining why it is positive on the vast majority of cells we measure (Section~\ref{sec:landscape} reports the small set of empirical negative cases); and \emph{max-superadditivity equivalence} states that a positive \tax\ is exactly the failure of the underlying cooperative game to satisfy $v(\{1, 2\}) \geq \max(v(\{1\}), v(\{2\}))$, supplying the cooperative-game vocabulary for the asymmetric pair effects of Section~\ref{sec:hetero}.

\paragraph{Operational form.}
For a task $T$ with instances $x$ scored by a deterministic grader $U \in [0,1]$, and a union-necessary partition $x = v_1(x) \cup v_2(x)$ such that neither view alone determines the answer, the homogeneous tax of model $M$ is
\begin{equation}
\widehat{\mathrm{tax}}(M, T) \;=\; s_{\text{solo-full}}(M, T) \;-\; s_{\text{homo}}(M, T),
\label{eq:tax}
\end{equation}
where $s_{\text{solo-full}}$ is the mean score of $M$ given the merged instance and $s_{\text{homo}}$ is the mean score of two copies of $M$ given $v_1$ and $v_2$ exchanging messages until termination, each averaged over $50$ rollouts with independent seeds. We additionally report the ratio version $\widehat{\mathrm{tax}}/s_{\text{solo-full}}$ to normalise across tasks with different solo ceilings. The heterogeneous case substitutes the stronger member's solo score for the first term.

\section{Tasks and Protocols}
\label{sec:tasks}

\subsection{Design Principles for the Task Suite}
\label{sec:tasks:principles}

For Equation~\ref{eq:tax} to cleanly measure coordination cost, every task in the suite satisfies four properties, three of which we inherit from \citet{davidson2025collaboration}. \textbf{Solo-trivial:} with the full instance a single agent should solve the task at a high rate, so that $s_{\text{solo-full}}$ is near ceiling and the gap reflects coordination cost rather than problem-solving capacity. \textbf{Union-necessary:} each instance is partitioned into views $v_1, v_2$ with $v_1 \cup v_2 = x$ and neither view alone admits the canonical answer. \textbf{Algorithmically verifiable:} ground truth is computed by a deterministic procedure (BFS for shortest paths, topological sort for orderings, a SAT-style solver for constraint puzzles) yielding a continuous score in $[0,1]$, with path-task graders enumerating all sixteen origin / orientation / axis-order schemes following \citet{davidson2025collaboration}. \textbf{Multiply expressible:} the same content admits several equivalent surface representations (coordinate origins, axis orderings, naming conventions, ordinal directions, relational vocabularies), providing the grounding friction we aim to measure \citep{clark1991grounding,levinson2003space}. Each principle is required by Definition~\ref{def:tax} or Proposition~\ref{prop:nonneg}: without it the tax either degenerates to zero by construction, becomes ill-defined as a population quantity, or mixes coordination cost with solo problem-solving capacity. Appendix~\ref{app:theory:principles} formalises each requirement.

\subsection{Task Families}
\label{sec:tasks:families}

The suite is organised into three families, chosen for three reasons. First, they span the principal answer structures that LLMs are asked to produce in multi-agent deployments: sequences, relational queries, and constraint-satisfying assignments. Second, they cover three distinct sources of grounding friction in the sense of Section~\ref{sec:tasks:principles}: reference-frame alignment (Spatial), lexical and category alignment in the sense of \citet{brennan1996conceptual,pickering2004toward} (Relational), and indexing and ordinal alignment (CSP). Third, they expose different error-propagation patterns: a single misaligned step invalidates the rest of a path \citep{wang2026reasoning}, relational query errors stay local to their query, and constraint violations cascade through the assignment. Together they let us ask whether collaboration ability is a single capacity or a profile that varies with task structure, and whether per-family bottlenecks differ.
The full per-task answer types, grader rules, and coverage within each family are reported in Appendix~\ref{app:tasks}.

\begin{figure*}[t]
  \centering
  \includegraphics[width=\textwidth]{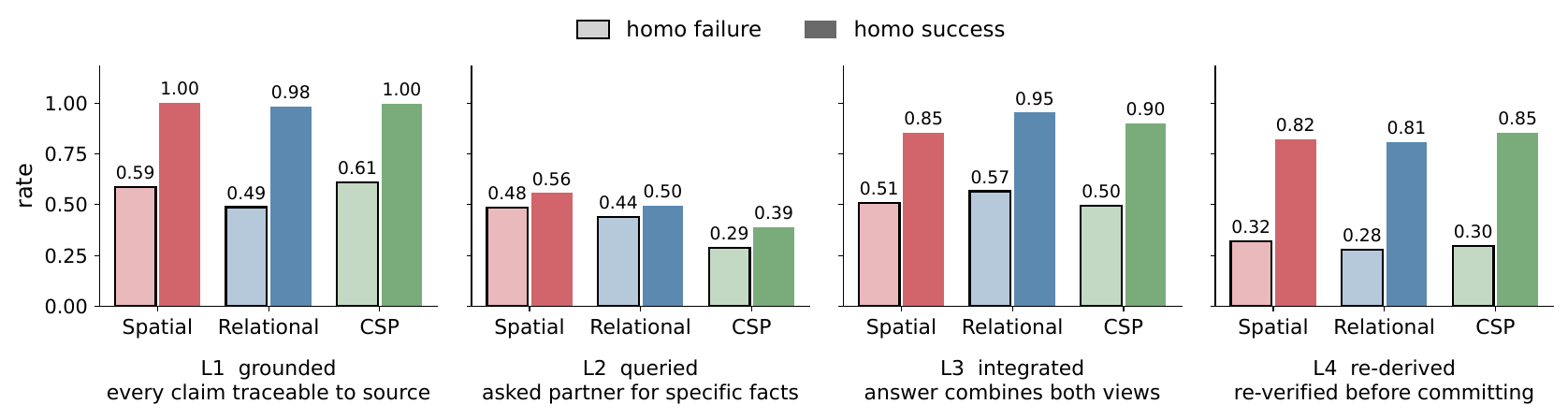}
  \caption{\textbf{Four-stage cascade by outcome and category.} Each panel reports the rate at which homogeneous failures (lighter bars) and successes (saturated bars) satisfy one of the four cascade dimensions, separated by task category. \emph{L1 grounded}: every claim is traceable to a stated source. \emph{L2 queried}: at least one agent asked the partner a specific factual question. \emph{L3 integrated}: the decisive claim was preceded by an explicit combined-state message. \emph{L4 re-derived}: the receiving agent showed recomputation work before either agent emitted \texttt{ACTI!}. Every stage discriminates failures from successes.}
  \label{fig:mechanism}
\end{figure*}

\subsection{Solo and Collaborative Modes}
\label{sec:protocols:modes}

We evaluate each task instance under four canonical modes. The first two are single-agent controls; the latter two are two-agent collaborations. In every mode the final answer is scored against the deterministic ground truth (pipeline described under Metrics below).

\paragraph{Solo control.} In \emph{solo-full} a single agent receives the merged instance $v_1 \cup v_2$ and produces a candidate answer in one completion. The candidate is then returned to the same model for a critic pass that may revise it before grading. The critic step pushes the tax toward pure coordination cost rather than a turn-count asymmetry between solo and collaboration: the collaboration mode receives multiple rounds of partner messages, each effectively acting as another pass over the answer, so giving solo a single revision opportunity matches that structural affordance. Empirically the critic contributes a small fraction of the solo score (Appendix~\ref{app:critic}), well below the gap to collaboration on the same model. Solo-full is the coordination-free upper bound and serves as the denominator of the gap.

\paragraph{Homogeneous collaboration.} Two instances of the same model, instantiated with independent contexts, receive $v_1$ and $v_2$ respectively. They exchange messages under a shared system prompt that explains the split-view protocol; each message is prefixed with \texttt{[other agent]:} when delivered to its counterpart. The dialogue continues until either agent emits the termination marker \texttt{ACTI!} together with a candidate answer, or a turn cap of $50$ exchanges is reached. The candidate is then passed to the grader. 

\paragraph{Heterogeneous collaboration.} Identical to homogeneous, except that the two agent roles are filled by different models (or by different checkpoints of the same family). This mode is the basis of the pair-composition analysis in Section~\ref{sec:hetero}.

\paragraph{Metrics.} For each (task, mode, model or pair) we run $50$ rollouts with independent seeds and report the mean continuous score $s \in [0, 1]$. The LLM never assigns a score itself; it only translates free-form text into the structured answer that the deterministic program then scores (per-task grader rules in Appendix~\ref{app:tasks}). The grader reads the full transcript in every case, so a rollout that reaches the turn cap without \texttt{ACTI!} is scored on the best answer proposed anywhere in the dialogue rather than assigned an automatic zero; Appendix~\ref{app:turncap} reports how often this occurs and shows the tax is unchanged when such rollouts are dropped. The \tax\ of Equation~\ref{eq:tax} and its ratio variant are over the solo-full and homogeneous modes; the heterogeneous mode enters the analysis in Section~\ref{sec:hetero}.

\section{Experiment}
\label{sec:experiment}

\subsection{Experimental Setup}
\label{sec:setup}

We evaluate eleven models from seven providers: OpenAI (\texttt{gpt-5}, \texttt{gpt-5-nano}, \texttt{gpt-4.1-mini}, \texttt{gpt-4.1-nano}, \texttt{gpt-4o-mini}), Anthropic (\texttt{claude-sonnet-4-5}), Google (\texttt{gemini-2.5-flash-lite}), DeepSeek (\texttt{DeepSeek-V4-Pro}), and three open-weight models hosted through API endpoints: \texttt{Llama-4-Maverick} (Meta, mixture-of-experts), \texttt{Phi-4} (Microsoft), and \texttt{Qwen3-8B} (Alibaba). The set spans a wide capability range while limiting redundancy within any single family. All models are queried through OpenAI-compatible chat APIs with no fine-tuning, no parameter access, and no sampling control beyond temperature. Detailed hyperparameters (temperatures, rollout count, turn cap) are reported in Appendix~\ref{app:hparams}.

\subsection{The Gap Landscape}
\label{sec:landscape}
\label{sec:results}

\paragraph{The collaboration tax is structured by both model capability and task type.}
As shown in Figure~\ref{fig:landscape}, two patterns hold without exception across the eleven models. Within every row, the ordering is Spatial $\succ$ Relational $\succ$ CSP: spatial-coordination tasks lose the most from collaboration, relational queries lose less, and constraint-satisfaction tasks lose least. Across rows, the gap scales monotonically with model capability: the weakest models lose roughly half of their solo success to coordination, while the top-tier models are also affected. The three weakest rows come from three different model families, so the capability ordering is not a family-style artefact.

\begin{figure}[t]
  \centering
  \includegraphics[width=\columnwidth]{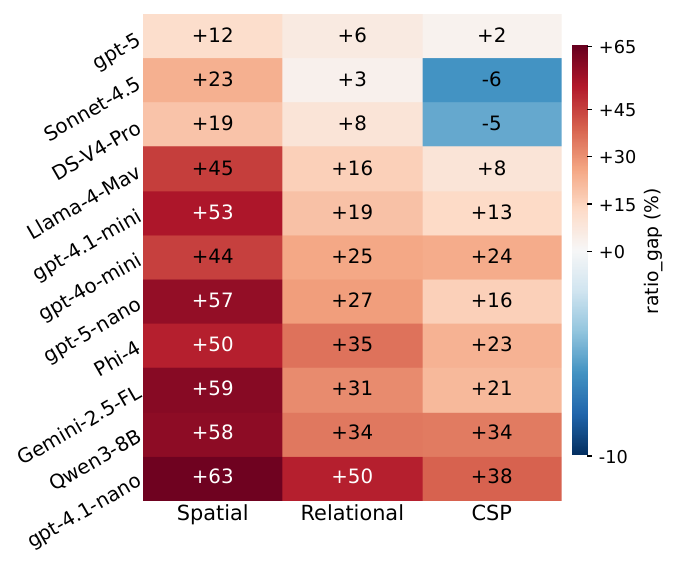}
  \caption{\textbf{Homogeneous ratio gap by model and category.} Each cell shows the mean ratio gap (in percentage points) for one (model, category) pair, with rows sorted by each model's overall mean. The gap is structured by both model capability and task category.}
  \label{fig:landscape}
\end{figure}

A per-task decomposition (Figure~\ref{fig:taskrank}) is reported in Appendix~\ref{app:results}.

\begin{figure}[t]
  \centering
  \includegraphics[width=\columnwidth]{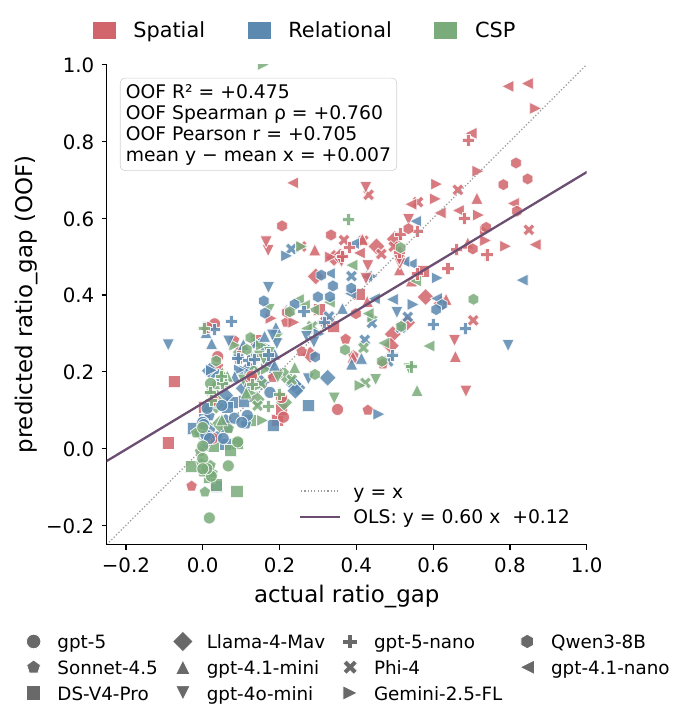}
  \caption{\textbf{Out-of-fold predicted vs actual ratio gap.} The M1 ridge regression fit on (model, task) cells with folds grouped by task. Out-of-fold $R^{2} = 0.475$, Spearman $\rho = 0.760$, Pearson $r = 0.705$.}
  \label{fig:predict}
\end{figure}

\subsection{Mechanism: A Four-Stage Cascade}
\label{sec:mechanism}

We next ask \emph{why} the gap arises. From $700$ open-ended LLM failure descriptions clustered under a neutral prompt and an anti-bias naming rule (Appendix~\ref{app:clusters}), we extract $16$ behaviourally specific themes and group them along the conversational-grounding pipeline of \citet{clark1991grounding,clark1996using}: claims (L1 Grounding), queries (L2 Querying), uptake (L3 Integration), and acknowledgment (L4 Re-derivation). Each stage is operationalised as a binary judge dimension applied to every homogeneous rollout. The four stages, from earliest to latest:

\paragraph{L1 Grounding.} A claim is grounded if it can be traced to information stated by either agent. Panel L1 of Figure~\ref{fig:mechanism} shows that grounding is the cleanest single-feature fail/success discriminator in the judge: successful rollouts are grounded in essentially every category, while a substantial fraction of failures contain at least one ungrounded claim. It also acts as a strict necessary condition: once an agent fabricates a value not stated by either view, no downstream stage can recover.

\paragraph{L2 Querying.} A pair queries iff at least one agent makes a specific factual request of the partner (\emph{``what is the value of node $K$?''}). Querying discriminates failure from success across all three categories (panel L2 of Figure~\ref{fig:mechanism}), with the largest gap on CSP, where successful pairs explicitly elicit cross-half capacity and constraint facts that failed pairs leave latent. Spatial and Relational pairs show smaller but non-zero querying gaps, because much information already arrives passively when agents dump their views. L2 is also the most independent of the four dimensions, the layer most likely to fire alone in a failed rollout.

\paragraph{L3 Integration.} A pair integrates iff the decisive claim is preceded by an explicit combined-state message that lists facts from both views and any derived consequences. Panel L3 of Figure~\ref{fig:mechanism} shows that integration is the strongest single-variable predictor of the \tax\ and the only stage whose marginal contribution to a multi-feature regression is positive. The effect peaks on Spatial, where path tasks carry persistent state and a single missed integration corrupts all subsequent moves.

\begin{table*}[t]
\centering
\small
\resizebox{\textwidth}{!}{%
\begin{tabular}{lcccccccc}
\toprule
& \multicolumn{2}{c}{Spatial} & \multicolumn{2}{c}{Relational} & \multicolumn{2}{c}{CSP} & \multicolumn{2}{c}{Overall} \\
\cmidrule(lr){2-3} \cmidrule(lr){4-5} \cmidrule(lr){6-7} \cmidrule(lr){8-9}
Condition & $\Delta s_{\text{homo}}$ & \% closed & $\Delta s_{\text{homo}}$ & \% closed & $\Delta s_{\text{homo}}$ & \% closed & $\Delta s_{\text{homo}}$ & \% closed \\
\midrule
All four & $+0.051_{\pm .026}$ & $+27.1_{\pm 12.3}$ & $\mathbf{+0.072}_{\pm .029}$ & $\mathbf{+41.4}_{\pm 38.0}$ & $+0.111_{\pm .074}$ & $+52.5_{\pm 32.3}$ & $\mathbf{+0.079}_{\pm .031}$ & $\mathbf{+38.3}_{\pm 14.1}$ \\
no L1    & $\mathbf{+0.066}_{\pm .027}$ & $\mathbf{+37.3}_{\pm 16.2}$ & $-0.005_{\pm .038}$ & $-2.4_{\pm 25.4}$ & $+0.093_{\pm .065}$ & $+45.6_{\pm 35.3}$ & $+0.051_{\pm .033}$ & $+27.7_{\pm 15.1}$ \\
no L2    & $+0.034_{\pm .048}$ & $+12.1_{\pm 26.0}$ & $+0.030_{\pm .030}$ & $+18.4_{\pm 21.2}$ & $+0.055_{\pm .028}$ & $+28.2_{\pm 20.5}$ & $+0.040_{\pm .022}$ & $+18.4_{\pm 13.4}$ \\
no L3    & $+0.059_{\pm .018}$ & $+30.3_{\pm 9.9}$ & $+0.044_{\pm .042}$ & $+27.5_{\pm 36.8}$ & $\mathbf{+0.119}_{\pm .078}$ & $\mathbf{+55.2}_{\pm 22.1}$ & $+0.075_{\pm .035}$ & $+36.3_{\pm 13.5}$ \\
no L4    & $-0.004_{\pm .027}$ & $+2.3_{\pm 12.5}$ & $+0.015_{\pm .035}$ & $+2.8_{\pm 37.5}$ & $+0.064_{\pm .063}$ & $+29.4_{\pm 25.6}$ & $+0.027_{\pm .030}$ & $+9.9_{\pm 13.4}$ \\
\bottomrule
\end{tabular}%
}
\caption{\textbf{Stage-targeted prompt-intervention ablation.} The \emph{all four} condition adds the L1 grounding, L2 query, L3 integration, and L4 re-derivation clauses to the system prompt; the four \emph{no~L$k$} rows are leave-one-out variants that drop the clause for layer $k$. $\Delta s_{\text{homo}}$ is the absolute lift in the homogeneous success rate against the no-intervention baseline; \emph{\% closed} is the fraction of the original tax that the intervention recovers. Subscripts give the half-width of the $95\%$ confidence interval. Bold marks the column-wise point-estimate maximum.}
\label{tab:intervention}
\end{table*}

\paragraph{L4 Re-derivation.} A pair re-derives iff the receiving agent shows actual recomputation work (re-walks the path, recomputes the sum, re-checks the constraints) before either agent wants to end. Re-derivation tracks the ratio gap almost as strongly as integration in absolute correlation, but its marginal contribution beyond integration is essentially zero: failures of L3 and L4 co-occur in the majority of failures, since the receiver cannot easily re-derive when the proposer never integrated. Whether this reflects redundancy or observational entanglement is tested by the intervention in Section~\ref{sec:intervention}.

\paragraph{Cascade structure.} The four stages are separable but not independent: in most failed rollouts at least two stages fire simultaneously, and L3 and L4 are entangled in baseline data, with full co-occurrence statistics in Appendix~\ref{app:cascade} and one illustrative single-stage failure per dimension in Appendix~\ref{app:cases}. To validate the judge labels we re-annotated $100$ stratified homogeneous rollouts with four independent expert annotators under a shared guideline: inter-rater agreement is moderate to substantial across all four stages (Fleiss' $\kappa$ in $[0.40, 0.69]$), and agreement between the automated judge and the expert majority is substantial on L1, L2, and L3 (Cohen's $\kappa$ in $[0.67, 0.73]$) and moderate on L4 ($\kappa = 0.51$) (Appendix~\ref{app:judge}). We test the implied prediction, that stage-targeted prompt interventions produce a separable lift per stage, in Section~\ref{sec:intervention}.

\subsection{Predicting the Gap}
\label{sec:predict}

If the gap is driven by mechanical conversational behaviours, it should be predictable from those behaviours. We fit a ridge regression (Ridge, $\lambda=1$, inputs standardised) of the per-cell ratio gap on a panel of conversation-structural features (sample size, length, agreement and disagreement marker density, view-disclosure ratios, coordinate-token density, and turn-pacing statistics) together with model and category dummies. To prevent leakage, folds are grouped by task: every held-out fold contains tasks not seen during fitting.

\paragraph{The collaboration tax is mechanically predictable from conversation features.}
As shown in Figure~\ref{fig:predict}, the regression achieves a substantial out-of-fold $R^{2}$ across the (model, task) cells, with strong held-out rank correlation throughout. Adding the cascade judge labels to the regression does not meaningfully improve the fit beyond integration (L3), indicating that the structural features already encode most of the predictive signal carried by the judge labels and that L4 collinearity with L3 leaves little additional variance to explain. The mapping is also not a same-model artefact: fit on homogeneous cells alone, it transfers to unseen heterogeneous cells at $R^{2} = 0.41$ (Appendix~\ref{app:results}).

\paragraph{Capability sets the intercept; conversation shape sets the slope.}
The regression generalises across tasks: a leave-one-task-out evaluation, in which every task is held out in turn while fitting on the rest, retains positive held-out variance explained and strong rank correlation, with the majority of held-out tasks individually positive. By contrast, the regression cannot extrapolate the absolute gap level to a held-out model: a leave-one-model-out variant preserves the rank ordering of cells but not their absolute level. The coefficients in Appendix Figure~\ref{fig:forest} show that the largest positive contributors are model dummies for the weakest models, while the largest negative contributors are the CSP dummy and conversation-shape features. We read this as a two-component decomposition of the gap: a model-level intercept set by base capability and a slope along conversation-shape features shared across models.

\subsection{Intervention: Stage-Targeted Prompt Clauses}
\label{sec:intervention}

If the four cascade stages of Section~\ref{sec:mechanism} are genuine mechanisms rather than correlates of failure, then prompt clauses targeted at each stage should produce a measurable lift in the homogeneous success rate, and the relative magnitude across stages should track the per-category fail rates of Section~\ref{sec:mechanism}. We test five conditions on top of the homogeneous baseline of Section~\ref{sec:setup}: an \emph{all-four} condition that appends a stage-specific clause for every layer to the system prompt (a grounding clause for L1, a query mandate for L2, an integration block for L3, and a re-derivation requirement for L4), and four leave-one-out variants \emph{no~L$k$} that drop the $k$-th clause while keeping the other three. The four clauses are listed verbatim in Appendix~\ref{app:prompts}.

\paragraph{The combined intervention recovers a substantial fraction of the tax.}
As shown in Table~\ref{tab:intervention} (last column), the all-four condition lifts homogeneous success against the no-intervention baseline, with the $95\%$ confidence interval bounded well above zero on every category. Per-category responsiveness follows the predictive signal of Section~\ref{sec:predict}: CSP responds most, Relational next, and Spatial least. No condition reaches the solo ceiling, but a single change to the system prompt recovers a substantial fraction of the entire collaboration tax across the suite, with no retraining and no change to the underlying model.

\paragraph{Each category is bottlenecked by a different cascade layer.}
Reading down a column of leave-one-out values isolates the marginal contribution of the dropped clause. As Table~\ref{tab:intervention} shows, dropping L4 on Spatial, L1 on Relational, or L2 on CSP each substantially reduces the lift in the respective category, with the $95\%$ confidence interval crossing zero for \emph{no~L4} on Spatial and for \emph{no~L1} on Relational; the critical layer differs across categories. The Spatial result is the strongest evidence that L4 is causally separable from L3: in the observational data of Section~\ref{sec:mechanism} the two co-occur and L4 carries no separable regression weight beyond L3, but once the L3 clause externally enforces integration, L4 carries the Spatial bottleneck on its own. In two cases, \emph{no~L1} on Spatial and \emph{no~L3} on CSP, the targeted leave-one-out point estimate exceeds the all-four point estimate, though the confidence intervals overlap heavily, indicating at most nominal headroom for category-selective prompts over the universal one.

\begin{figure}[t]
  \centering
  \includegraphics[width=\columnwidth]{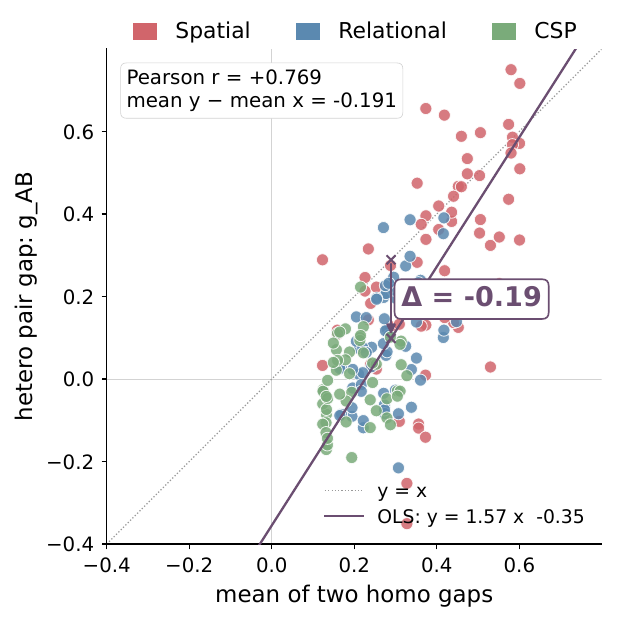}
  \caption{\textbf{Additive-null scatter for heterogeneous pairs.} Each point is one (pair, task) cell: $x$ is the average of the two participating models' homogeneous gaps; $y$ is the actual hetero gap. The OLS fit ($y = 1.57x - 0.35$) and the $y = x$ reference line are shown. Hetero gaps fall below the additive midpoint.}
  \label{fig:additivenull}
\end{figure}

\section{Heterogeneous-Pair Matrix}
\label{sec:hetero}

We now ask whether the same gap structure holds when the two agents are different models. Running the full $11 \times 11$ pair matrix is too costly for the question we care about, so we instead pick three strong-weak contrasts that hold the capability gap roughly fixed while varying family identity. All three use \texttt{gpt-4.1-nano} as the weak slot: one pairs it with \texttt{claude-sonnet-4-5} (cross-family, Anthropic and OpenAI), one with \texttt{gpt-5} (same-family, both OpenAI), and one with \texttt{DeepSeek-V4-Pro} (cross-family, DeepSeek and OpenAI). For each pair we run agent-1 / agent-2 in both orderings, giving six heterogeneous configurations on the same task suite. With three pairs and six configurations, our heterogeneous results are reported as an existence proof of a qualitative pattern rather than a quantitative characterisation; we use them to probe one structural prediction of the cooperative-game framing of Section~\ref{sec:tax}.

\paragraph{The pair gap is pulled toward the stronger member, not toward the midpoint.}
Aggregating across (pair, task) cells (Figure~\ref{fig:additivenull}), the actual hetero ratio gap correlates strongly with the midpoint between the two individual homogeneous gaps but is systematically below it and well below the additive line $y = x$; the per-cell breakdown by (initiator, responder) is reported in Appendix Figure~\ref{fig:heteromat}, where the six off-diagonal heterogeneous cells cluster near the strong-tier diagonal rather than averaging between strong and weak. Equivalently, all six configurations violate \emph{max-superadditivity} \citep{shapley1953value} on the majority of tasks, and Table~\ref{tab:shapley} translates this into cooperative-game terms: every member's Shapley share falls below its singleton payoff, as Appendix~\ref{app:theory:hetero} formalises. The two-player Shapley allocation reported there is
\begin{equation}
\phi_{i} = \tfrac{1}{2}\bigl(v(\{i\}) + v(\{1, 2\}) - v(\{j\})\bigr), \quad j \neq i,
\label{eq:shapley_main}
\end{equation}
with $v(\{i\})$ each member's homogeneous self-pair mean and $v(\{1, 2\})$ the pair's mean. Same-family and cross-family pairs are not visibly separated; with one same-family pair, family-level effects are not detectable.
Two further robustness checks are reported in Appendix~\ref{app:results}.
\begin{table}[t]
\centering
\small
\resizebox{\columnwidth}{!}{%
\begin{tabular}{lccccc}
\toprule
Pair & $v(\{1\})$ & $v(\{2\})$ & $v(\{1{,}2\})$ & $\phi_{1}$ & $\phi_{2}$ \\
\midrule
\texttt{Sonnet} $\times$ \texttt{nano} & $0.869$ & $0.352$ & $0.731$ & $0.624$ & $0.107$ \\
\texttt{nano} $\times$ \texttt{Sonnet} & $0.352$ & $0.869$ & $0.732$ & $0.108$ & $0.624$ \\
\texttt{nano} $\times$ \texttt{gpt-5}  & $0.352$ & $0.932$ & $0.756$ & $0.088$ & $0.668$ \\
\texttt{gpt-5} $\times$ \texttt{nano}  & $0.932$ & $0.352$ & $0.706$ & $0.643$ & $0.063$ \\
\texttt{DeepSeek} $\times$ \texttt{nano} & $0.802$ & $0.352$ & $0.728$ & $0.589$ & $0.139$ \\
\texttt{nano} $\times$ \texttt{DeepSeek} & $0.352$ & $0.802$ & $0.718$ & $0.134$ & $0.584$ \\
\bottomrule
\end{tabular}%
}
\caption{\textbf{Cooperative-game values and Shapley allocations.} Characteristic-function values $v(\cdot)$ and Shapley values $\phi$ for the six heterogeneous pair configurations, averaged across tasks. Every pair violates max-superadditivity, and every member's Shapley value falls below its singleton payoff.}
\label{tab:shapley}
\end{table}

\section{Conclusion}
\label{sec:conclusion}

We have formulated the \emph{\tax} as the team-decentralisation loss of \citet{marschak1955elements} for two-player cooperative games with private information, and measured it across $32$ tasks and $11$ models. The tax is structured along two no-exception orderings (capability and task category), driven by a four-stage conversational cascade rather than a reasoning deficit, mechanically predictable from conversation features, partly tractable through a single stage-targeted prompt intervention, and pulled toward the stronger member in heterogeneous pairs. These results recast LLM coordination cost as a structured, predictable, and partly tractable phenomenon, and they suggest that practical multi-agent deployments can recover a substantial fraction of solo capacity through prompt-level changes alone, without retraining or pair-specific tuning.

\section*{Acknowledgments}
This work was partially supported by the NSF under grants IIS-2533550, IIS-2321504, IIS-2217239, CNS-2426514, and CMMI-2146076, Notre Dame Strategic Framework Research Grant (2025), and Notre Dame Poverty Research Package (2025). Any expressed opinions, findings, and conclusions or recommendations are those of the authors and do not necessarily reflect the views of the sponsors.

\section*{Limitations}
\label{sec:limitations}

\paragraph{Two-agent only.} Our suite measures the \tax\ for dyadic pairs ($N=2$). Whether the four-stage cascade and the rank-survives, level-fails decomposition generalise to $N \geq 3$ multi-agent settings is unstudied; both the partition structure (now $N$ private views and $\binom{N}{2}$ pairwise channels) and the conversational dynamics (now a multi-party dialogue with floor management) change qualitatively beyond the dyadic case. We are extending the framework to multi-agent collaboration as ongoing work.

\paragraph{Synthetic tasks rather than deployment workloads.} Every task in our suite is a procedurally generated puzzle (grid, graph, constraint satisfaction). This isolates coordination mechanics from domain knowledge and makes the \tax\ cleanly attributable to grounding friction, but the resulting setting does not directly correspond to the multi-agent applications LLMs are increasingly deployed in \citep{li2025food4all,li2026hoosierhelp,wang2026mdforge}: collaborative code editing, issue triage and resolution, multi-turn debugging, or interactive document drafting, where the answer space and the feedback signal are far less structured. Future work will design tasks closer to these production workloads (paired code-modification dialogues, issue-resolution pairs over a shared repository) and re-measure the \tax\ under those conditions.

\section*{Ethics Statement}
\label{sec:ethics}

This paper involves no human subjects, no scraping of user data, and no demographic targeting. All experiments use procedurally generated puzzle instances and commercial LLM APIs with no fine-tuning. Released artifacts (task generators, splitters, graders, prompts, and anonymised conversation transcripts) enable reproduction and follow-up analysis. The mechanistic insights into where multi-agent LLM coordination breaks down are intended to support more robust collaborative systems; we are not aware of dual-use risks specific to this work beyond those already present in publicly available LLM agent benchmarks.

\bibliography{custom}

\appendix

\section{Theoretical Setup}
\label{app:theory}

This appendix gives the formal Definition of the \tax\ summarised in Section~\ref{sec:tax}, the cooperative-game formalisation underlying it, and proofs of the two Propositions invoked there, plus two further propositions on Shapley value allocation that justify the asymmetric pair effects discussed in Section~\ref{sec:hetero}.

\subsection{Formal definition}
\label{app:theory:def}

\begin{definition}[Collaboration Tax]
\label{def:tax}
Let $T = (\mathcal{X}, \mu, U)$ be a task specified by an instance distribution $\mu$ over an instance space $\mathcal{X}$ and a deterministic grader $U: \mathcal{X} \times \mathcal{A} \to [0, 1]$ that maps an instance and an answer to a continuous score. Let $x \mapsto (v_{1}(x), v_{2}(x))$ be a union-necessary partition satisfying $v_{1}(x) \cup v_{2}(x) = x$ for every $x \in \mathcal{X}$. For a pair of language-model agents $(M_{1}, M_{2})$ executing a paired protocol $\Pi$ that produces a candidate answer $\Pi(M_{1}, M_{2}; v_{1}(x), v_{2}(x)) \in \mathcal{A}$ from the private views, and writing $M(x)$ for the answer produced by a single instance of model $M$ given the full instance $x$, the \emph{\tax} of the pair on task $T$ under protocol $\Pi$ is
\begin{align}
&V_{\text{solo}} \;:=\; \max_{i \in \{1, 2\}} \mathbb{E}_{x \sim \mu}\bigl[U(x, M_{i}(x))\bigr], \label{eq:tax_solo} \\
&V_{\text{pair}} \;:=\; \mathbb{E}_{x \sim \mu}\bigl[U\bigl(x, \Pi(M_{1}, M_{2}; v_{1}(x), v_{2}(x))\bigr)\bigr], \label{eq:tax_pair} \\
&\mathrm{tax}\bigl((M_{1}, M_{2}), T\bigr) \;=\; V_{\text{solo}} - V_{\text{pair}}. \label{eq:tax_def}
\end{align}
For homogeneous pairs $M_{1} = M_{2} = M$ the maximum reduces to $\mathbb{E}_{x}[U(x, M(x))]$, and we write $\mathrm{tax}(M, T)$.
\end{definition}

The operational estimator $\widehat{\mathrm{tax}}$ of Equation~\ref{eq:tax} in Section~\ref{sec:tax} replaces $V_{\text{solo}}$ and $V_{\text{pair}}$ by their cell-means over $50$ rollouts per cell with independent seeds.

\subsection{Cooperative team formulation}
\label{app:theory:formulation}

Following \citet{marschak1955elements}, we model each paired execution as a $2$-agent cooperative team problem with private information. Let $(N, v)$ denote a cooperative game with player set $N = \{1, 2\}$ and characteristic function $v: 2^{N} \to [0, 1]$ given by
\begin{align}
v(\emptyset) &= 0, \\
v(\{i\}) &= \mathbb{E}_{x \sim \mu}\bigl[U(x, M_{i}(x))\bigr], \\
v(\{1, 2\}) &= \mathbb{E}_{x \sim \mu}\bigl[U(x, \Pi(M_{1}, M_{2}; v_{1}(x), v_{2}(x)))\bigr].
\end{align}
By the union-necessary property of the partition introduced in Section~\ref{sec:tasks:principles}, $v(\{i\})$ is realised by a single agent given the full instance $x = v_{1}(x) \cup v_{2}(x)$, and $v(\{1, 2\})$ is realised by the paired protocol on the same instance distribution. In this notation, Definition~\ref{def:tax} reads
\begin{equation}
\begin{split}
&\mathrm{tax}\bigl((M_{1}, M_{2}), T\bigr) \;=\; \\
\max\bigl(v(\{&1\}), v(\{2\})\bigr) \;-\; v(\{1, 2\}).
\end{split}
\label{eq:tax_v}
\end{equation}

\subsection{Information dominance}
\label{app:theory:nonneg}

\begin{proposition}[Information dominance]
\label{prop:nonneg}
If a paired protocol $\Pi$ produces an answer distribution that is measurable with respect to $(v_{1}(x), v_{2}(x))$, and $V_{\text{solo}}$ in Definition~\ref{def:tax} is realised by a policy that is optimal among all measurable functions of $x$, then $\mathrm{tax}((M_{1}, M_{2}), T) \geq 0$.
\end{proposition}

\begin{proof}
Let $\pi^{\Pi}: \mathcal{X} \to \Delta(\mathcal{A})$ denote the answer distribution induced by the paired protocol, marginalised over the conversation, as a function of the joint private views $(v_{1}(x), v_{2}(x))$. By union-necessity, $(v_{1}(x), v_{2}(x))$ is a measurable function of $x$, so $\pi^{\Pi}$ is also a measurable function of $x$, and the expected payoff $\mathbb{E}_{x}[U(x, a) \mid a \sim \pi^{\Pi}(x)]$ is well-defined. Let $\pi^{\star} \in \arg\max_{\pi}\, \mathbb{E}_{x, a \sim \pi(x)}[U(x, a)]$ over all measurable policies $\pi: \mathcal{X} \to \Delta(\mathcal{A})$. Since $\pi^{\Pi}$ lies in the feasible set, we have $\mathbb{E}[U(x, \pi^{\star}(x))] \geq \mathbb{E}[U(x, \pi^{\Pi}(x))]$. Under the assumption in the proposition, $V_{\text{solo}}$ in Definition~\ref{def:tax} attains the supremum on the right-hand side, so $\mathrm{tax}((M_{1}, M_{2}), T) \geq 0$.
\end{proof}

The operational estimator $\widehat{\mathrm{tax}}$ in Equation~\ref{eq:tax} replaces the optimal centralised policy with the policy realised by a single LLM agent given the full instance. In general this single-agent policy is not the supremum over all measurable policies (an LLM with $x$ in context may still produce a suboptimal answer), and accordingly $\widehat{\mathrm{tax}}$ can take negative values when the paired protocol happens to outperform the single-agent baseline. Section~\ref{sec:landscape} reports the small set of such cells.

\subsection{Max-superadditivity equivalence}
\label{app:theory:hetero}

\begin{proposition}[Max-superadditivity]
\label{prop:maxsup}
For any pair $(M_{1}, M_{2})$ and task $T$, $\mathrm{tax}((M_{1}, M_{2}), T) > 0$ if and only if the cooperative game with value function $v(\{i\}) = \mathbb{E}[U(x, M_{i}(x))]$ and $v(\{1, 2\}) = \mathbb{E}[U(x, \Pi(M_{1}, M_{2}; v_{1}(x), v_{2}(x)))]$ violates max-superadditivity: $v(\{1, 2\}) < \max(v(\{1\}), v(\{2\}))$.
\end{proposition}

\begin{proof}
Substituting Equation~\ref{eq:tax_v}, $\mathrm{tax}((M_{1}, M_{2}), T) > 0$ is equivalent to $\max(v(\{1\}), v(\{2\})) - v(\{1, 2\}) > 0$, that is $v(\{1, 2\}) < \max(v(\{1\}), v(\{2\}))$. This is precisely the failure of max-superadditivity for the cooperative game $(N, v)$.
\end{proof}

A coalition that satisfies max-superadditivity produces at least as much joint utility as its strongest member acting alone. The empirical content of Proposition~\ref{prop:maxsup} is therefore the claim that LLM pairs across our six heterogeneous configurations \emph{fail} this property on the majority of tasks, as Table~\ref{tab:shapley} of Section~\ref{sec:hetero} reports.

\subsection{Shapley value and subadditivity}
\label{app:theory:shapley}

The Shapley value \citep{shapley1953value} provides an axiomatic allocation of the joint payoff $v(\{1, 2\})$ between the two members under efficiency, symmetry, and the null-player axiom.

\begin{proposition}[Shapley value of a 2-agent cooperative game]
\label{prop:shapley}
For the cooperative game $(N, v)$ with $N = \{1, 2\}$, the Shapley value of player $i$ is
\begin{equation}
\phi_{i} \;=\; \tfrac{1}{2}\Bigl(v(\{i\}) + v(\{1, 2\}) - v(\{j\})\Bigr), \qquad j \neq i.
\label{eq:shapley}
\end{equation}
\end{proposition}

\begin{proof}
The Shapley value of player $i$ is the average marginal contribution over orderings $\sigma$ of $N$:
\begin{equation*}
\phi_{i} \;=\; \frac{1}{|N|!} \sum_{\sigma} \bigl[v(S^{\sigma}_{i} \cup \{i\}) - v(S^{\sigma}_{i})\bigr],
\end{equation*}
where $S^{\sigma}_{i}$ is the set of players preceding $i$ in ordering $\sigma$. For $|N| = 2$ there are two orderings: $(i, j)$ contributes $v(\{i\}) - v(\emptyset) = v(\{i\})$, and $(j, i)$ contributes $v(\{1, 2\}) - v(\{j\})$. Averaging yields Equation~\ref{eq:shapley}.
\end{proof}

\begin{proposition}[Subadditivity implies both members lose Shapley share]
\label{prop:subadd}
If $v(\{1, 2\}) < v(\{1\}) + v(\{2\})$, then $\phi_{i} < v(\{i\})$ for both $i \in \{1, 2\}$.
\end{proposition}

\begin{proof}
From Equation~\ref{eq:shapley},
\begin{equation}
\phi_{i} - v(\{i\}) \;=\; \tfrac{1}{2}\bigl(v(\{1, 2\}) - v(\{1\}) - v(\{2\})\bigr),
\label{eq:shapley_delta}
\end{equation}
which is negative under the hypothesis. The right-hand side does not depend on $i$, so both members lose by the same absolute amount.
\end{proof}

Proposition~\ref{prop:subadd} is empirically applicable: all six of our heterogeneous pair configurations satisfy $v(\{1, 2\}) < v(\{1\}) + v(\{2\})$ on the majority of tasks, so both members' Shapley values are strictly below their singleton payoffs. Table~\ref{tab:shapley} reports the cooperative-game values and Shapley values averaged across the $32$ tasks. In every pair, the strong member's Shapley value falls substantially below its solo payoff (e.g., $\phi_{\text{gpt-5}} = 0.668$ versus $v(\{\text{gpt-5}\}) = 0.932$ on the \texttt{nano} $\times$ \texttt{gpt-5} configuration), illustrating that under Shapley fairness the stronger member's marginal contribution to the team is well below its solo capacity.

\subsection{Design principles as theoretical requirements}
\label{app:theory:principles}

We sketch why each of the four design principles of Section~\ref{sec:tasks:principles} is required by Definition~\ref{def:tax} or Proposition~\ref{prop:nonneg}, in the sense that violating it forces the tax to be uninformative, ill-defined, or zero by construction.

\paragraph{Solo-trivial.} When $V_{\text{solo}}$ is far from ceiling, both $V_{\text{solo}}$ and $V_{\text{pair}}$ are bounded above by problem-solving capacity rather than by coordination quality. The difference $V_{\text{solo}} - V_{\text{pair}}$ then aggregates two unrelated sources of failure, and an empirical $\widehat{\mathrm{tax}}$ near zero is indistinguishable from a regime where both the solo agent and the pair fail for problem-solving reasons. Tuning the task so that $V_{\text{solo}}$ is near ceiling restores the interpretation of the tax as a clean coordination signal.

\paragraph{Union-necessary.} Definition~\ref{def:tax} assumes a partition $v_1 \cup v_2 = x$ with neither view alone realising the canonical answer. If the partition is degenerate, say $v_1 = x$, then the paired protocol $\Pi$ can ignore $v_2$ and emulate the solo policy exactly, producing $V_{\text{pair}} = V_{\text{solo}}$ and forcing $\mathrm{tax} = 0$ by construction regardless of coordination ability. Union-necessity is the minimum partition non-degeneracy that gives the tax a chance to be positive.

\paragraph{Algorithmically verifiable.} Definition~\ref{def:tax} requires $U: \mathcal{X} \times \mathcal{A} \to [0,1]$ to be a deterministic function, so that $V_{\text{solo}}$ and $V_{\text{pair}}$ are population quantities and the cell-level estimator over $50$ rollouts has variance only from agent stochasticity. A stochastic grader adds an independent noise source to both terms that is not absorbed by the within-cell averaging, so the estimator becomes inconsistent for the population $\mathrm{tax}$ at any finite sample size.

\paragraph{Multiply expressible.} Proposition~\ref{prop:nonneg} states $\mathrm{tax} \geq 0$ under information dominance; the upper bound $V_{\text{pair}} = V_{\text{solo}}$ is attainable in principle whenever $\Pi$ can recover the full instance from the views. If the views are trivially mergeable, for instance one agent serialising its view to the other in a canonical form that both agents share, then $\Pi$ can directly emulate the centralised baseline and the upper bound binds with equality at zero coordination effort. Multiple equivalent surface representations break this trivial pass-through: agents grounded in different schemes (origin, axis order, naming convention) cannot simply concatenate their views without first aligning representations, so the upper bound is approached only by competent coordination, and the tax becomes a measure of that competence.

\subsection{Remarks}
\label{app:theory:remarks}

\paragraph{Negative tax exceptions.} For sufficiently strong models on CSP tasks the homogeneous tax can be negative; \texttt{claude-sonnet-4-5} and \texttt{DeepSeek-V4-Pro} on CSP are the two cases reported in Section~\ref{sec:landscape}. This violates the classical $\mathrm{tax} \geq 0$ inequality and reflects the fact that LLM policies are stochastic and context-dependent: an agent's effective decision distribution can differ between solo and paired modes, breaking the policy-consistency assumption \citet{marschak1955elements} invokes when proving non-negativity.

\paragraph{Bounded protocol.} The classical equality condition (tax $= 0$ iff communication is sufficient to reconstruct $v_{1} \cup v_{2}$) requires unbounded message exchange. Our protocol caps the dialogue at $50$ turns, as Appendix~\ref{app:hparams} describes; we conjecture that part of the residual tax under the all-four prompt intervention of Section~\ref{sec:intervention} is attributable to this bound.

\paragraph{Scope.} The formalism here applies to $2$-agent pairs with a single common payoff and union-necessary information. Extensions to $n \geq 3$ agents, weighted contributions, or non-cooperative settings are outside the scope of this paper.

\section{Failure-Description Clustering}
\label{app:clusters}

This appendix details the pipeline that produced the sixteen failure themes referenced in Section~\ref{sec:mechanism} and the cluster-to-cascade-stage mapping that motivated the four-stage judge.

\paragraph{Pipeline.} For each failed homogeneous rollout we ask an LLM (\texttt{gpt-4o-mini}, temperature $0.2$) for an open-ended description of what went wrong in the conversation, under a strictly neutral prompt that does not pre-list candidate failure types and does not name any stage of the protocol. We collect $700$ such descriptions covering all eleven models and all three task categories, embed them with TF-IDF on lower-cased uni- and bi-grams, and cluster the embeddings with cosine $k$-means. The number of clusters, $k = 16$, is chosen as the smallest setting under which every cluster is dominated by a single task family (Spatial, Relational, or CSP); beyond $k = 16$ further splits produce sub-clusters within a single task type rather than new failure mechanisms.

\paragraph{Anti-bias naming rule.} We then ask the same LLM to produce a short, behaviourally specific name for each cluster from its centroid descriptions, under a strict constraint: a black-list of generic words (\emph{verification}, \emph{verify}, \emph{pool}, \emph{share}, \emph{communicate}, \emph{collaborate}, \emph{cross-check}, \emph{calculation}, \emph{format}) is forbidden anywhere in the name. The rule prevents the namer from collapsing distinct mechanisms into the same vacuous descriptor (every spatial coordination failure being labelled ``verification'' would erase exactly the structure we are trying to surface) and forces names that refer to concrete observable behaviours, such as \texttt{moved\_without\_confirming\_layout} or \texttt{failed\_to\_clarify\_conflict\_constraints}.

\paragraph{Cluster-to-cascade mapping.} We inspected the $16$ cluster names and their centroid transcripts and grouped each cluster by which stage of the dialogue the failure occurs at, following the classical grounding pipeline of \citet{clark1991grounding,clark1996using}: claims that should be sourced from either view (L1 Grounding), queries that elicit missing information from the partner (L2 Querying), uptake that integrates the two views into a combined statement (L3 Integration), and acknowledgment that re-derives the joint claim before committing (L4 Re-derivation). The resulting per-cluster mapping is reported in Table~\ref{tab:clusters}. The four stages are encoded as four independent judge prompts that a separate LLM applies to every rollout without exposure to the cluster labels (Appendix~\ref{app:prompts}), and validated against three blind expert annotators in Appendix~\ref{app:judge}.

\begin{table*}[h]
\centering
\small
\begin{tabular}{l l r l l}
\toprule
Stage & Cluster name & $n$ & Top task & Category \\
\midrule
\multirow{7}{*}{L1 Grounding}
& \texttt{stated\_incomplete\_attribute\_info}            & $78$  & \texttt{attrcombo}    & Relational \\
& \texttt{stated\_incorrect\_counts\_without\_confirmation} & $64$ & \texttt{graphattr}   & Relational \\
& \texttt{stated\_incorrect\_availability\_times}         & $34$  & \texttt{seating}      & CSP        \\
& \texttt{failed\_weight\_confirmation}                   & $32$  & \texttt{graph}        & Relational \\
& \texttt{assumed\_correct\_mappings\_without\_checking}  & $29$  & \texttt{cipherdecode} & CSP        \\
& \texttt{assumed\_incorrect\_word\_fit}                  & $24$  & \texttt{crossword}    & CSP        \\
& \texttt{assumed\_pairs\_without\_validation}            & $19$  & \texttt{numberlink}   & Spatial    \\
\midrule
\multirow{2}{*}{L2 Querying}
& \texttt{failed\_to\_confirm\_wall\_layout}              & $48$  & \texttt{geomaze}        & Spatial \\
& \texttt{failed\_to\_clarify\_conflict\_constraints}     & $33$  & \texttt{meetingmatrix}  & CSP     \\
\midrule
\multirow{5}{*}{L3 Integration}
& \texttt{moved\_without\_confirming\_layout}             & $109$ & \texttt{cube}        & Spatial \\
& \texttt{proposed\_order\_without\_full\_analysis}       & $37$  & \texttt{order}       & CSP     \\
& \texttt{stated\_status\_without\_confirmation}          & $33$  & \texttt{floodfill}   & Spatial \\
& \texttt{ignored\_blocked\_cells\_in\_path}              & $30$  & \texttt{timedpath}   & Spatial \\
& \texttt{failed\_capacity\_check\_before\_assignment}    & $25$  & \texttt{binpack}     & CSP     \\
\midrule
\multirow{2}{*}{L4 Re-derivation}
& \texttt{confirmed\_incorrect\_paths\_without\_checking} & $56$  & \texttt{treesum}      & Relational \\
& \texttt{failed\_to\_verify\_triangle\_existence}        & $49$  & \texttt{trianglecount}& Relational \\
\bottomrule
\end{tabular}
\caption{\textbf{Cluster-to-cascade-stage mapping.} For each of the $16$ failure themes: cluster name, number of descriptions, top task, dominant category. Centroid examples and full per-cluster distributions are released with the code.}
\label{tab:clusters}
\end{table*}

\section{Human Validation of the Cascade Judge}
\label{app:judge}

We validate the four-stage cascade judge with four independent expert annotators re-labelling a stratified sample of homogeneous rollouts under a shared guideline. We report per-dimension Fleiss' $\kappa$ across the four experts, pairwise Cohen's $\kappa$ for every expert pair, and Cohen's $\kappa$ between the automated cascade judge of Section~\ref{sec:mechanism} and the majority vote of the four experts.

\paragraph{Sample.} $100$ rollouts stratified jointly on outcome and category (50 failures, 50 successes; roughly $33$ per task category; all $11$ models and $31$ of $32$ tasks represented). Per-dimension stratification on the automated judge ensures both label values appear in the sample for every dimension. All annotators worked blind to model identity, rollout score, and the automated judge's label.

\paragraph{Inter-annotator agreement.} Agreement is moderate to substantial across all four stages, with L2 and L3 substantial and L1 and L4 moderate; L1 is the lowest because identifying every fabricated atomic value in long deductive transcripts is intrinsically threshold-sensitive.

\begin{table}[h]
\centering
\small
\begin{tabular}{l c}
\toprule
Dimension & Fleiss' $\kappa$ \\
\midrule
L1 grounded     & $0.398$ \\
L2 queried      & $0.650$ \\
L3 integrated   & $\mathbf{0.694}$ \\
L4 re-derived   & $0.567$ \\
\midrule
mean            & $0.577$ \\
\bottomrule
\end{tabular}
\caption{\textbf{Inter-annotator agreement across four expert annotators.} Landis-Koch: $\geq 0.40$ moderate, $\geq 0.60$ substantial. Bold marks the highest dimension.}
\label{tab:judge_kappa}
\end{table}

Pairwise Cohen's $\kappa$ for the six expert pairs is reported in Table~\ref{tab:pairwise_kappa}. The closest pair (E1 $\times$ E2) reaches almost-perfect agreement on L3 ($\kappa = 0.832$) and substantial agreement on L2 ($0.772$); pairwise means across the six pairs closely track the Fleiss' $\kappa$ values in Table~\ref{tab:judge_kappa}.

\begin{table}[h]
\centering
\small
\begin{tabular}{lcccc}
\toprule
Pair & L1 & L2 & L3 & L4 \\
\midrule
E1 $\times$ E2 & $\mathbf{0.494}$ & $\mathbf{0.772}$ & $\mathbf{0.832}$ & $0.544$ \\
E1 $\times$ E3 & $0.342$ & $0.583$ & $0.646$ & $\mathbf{0.696}$ \\
E1 $\times$ E4 & $0.398$ & $0.659$ & $0.739$ & $0.666$ \\
E2 $\times$ E3 & $0.400$ & $0.683$ & $0.608$ & $0.584$ \\
E2 $\times$ E4 & $0.458$ & $0.764$ & $0.699$ & $0.368$ \\
E3 $\times$ E4 & $0.306$ & $0.467$ & $0.661$ & $0.654$ \\
\midrule
mean & $0.400$ & $0.655$ & $0.698$ & $0.585$ \\
\bottomrule
\end{tabular}
\caption{\textbf{Pairwise Cohen's $\kappa$ across the six expert pairs.} Bold marks the column-wise maximum. Pair means closely match the Fleiss' $\kappa$ values in Table~\ref{tab:judge_kappa}.}
\label{tab:pairwise_kappa}
\end{table}

\paragraph{Automated-judge alignment.} Cohen's $\kappa$ between the automated cascade judge and the majority vote of the four experts is substantial on L1 ($\kappa = 0.707$), L2 ($0.668$), and L3 ($0.725$), and moderate on L4 ($0.508$). Judge-vs-majority agreement on L1 ($0.707$) exceeds inter-annotator agreement on the same dimension ($0.398$), consistent with majority voting denoising individual-annotator threshold variation; this supports the use of the automated cascade labels in Section~\ref{sec:mechanism} and Figure~\ref{fig:mechanism}.

\paragraph{Guideline.} The shared guideline contains: (i) verbatim L1 to L4 definitions from Section~\ref{sec:mechanism}; (ii) a list of common boundary cases, for example that an incorrect final answer is not by itself an L1 violation; (iii) an independence reminder that each dimension must be judged separately, without one stage's decision contaminating another; and (iv) several worked examples drawn from real transcripts that cover the frequent edge cases on each dimension. The guideline document is released with the code.

\section{Cascade Co-occurrence}
\label{app:cascade}

This appendix reports the joint firing pattern of the four cascade stages on the failed homogeneous rollouts. A stage is said to \emph{fire} on a rollout when the corresponding judge dimension is $0$ (the mechanism is absent).

\paragraph{Most failures fire multiple stages simultaneously.} Table~\ref{tab:cascade_count} shows the distribution of the number of stages firing per failed rollout. $67\%$ of failed rollouts fire at least two stages, justifying the §\ref{sec:mechanism} claim that the four stages are separable but not independent.

\begin{table}[h]
\centering
\small
\begin{tabular}{cc}
\toprule
\# stages firing & \% of failed rollouts \\
\midrule
$0$ & $12.8\%$ \\
$1$ & $20.2\%$ \\
$2$ & $22.5\%$ \\
$3$ & $22.7\%$ \\
$4$ & $21.8\%$ \\
\bottomrule
\end{tabular}
\caption{\textbf{Number of cascade stages firing per failed rollout.} Distribution over failed homogeneous rollouts.}
\label{tab:cascade_count}
\end{table}

\paragraph{L3 and L4 are the most entangled pair; L2 is the most independent.} Table~\ref{tab:cascade_jaccard} reports pairwise Jaccard overlap (intersection / union of failure sets) for every pair of cascade dimensions. L3 and L4 co-fire at Jaccard $0.67$, the largest overlap by a wide margin, which motivates the observational-vs-causal reconciliation in §\ref{sec:mechanism} and §\ref{sec:intervention}. L2 has the lowest average pairwise Jaccard with the other three stages and the highest sole-fire rate ($13.9\%$ of failed rollouts fire L2 and no other stage), confirming L2 as the most independent of the four dimensions. L3 never fires alone in our data ($0\%$ sole-fire), consistent with L3 sitting late in the cascade so that an integration failure is virtually always accompanied by an earlier grounding, query, or re-derivation failure as well.

\begin{table}[h]
\centering
\small
\begin{tabular}{lcccc}
\toprule
       & L1 & L2 & L3 & L4 \\
\midrule
L1 &        & $0.37$ & $0.54$ & $0.56$ \\
L2 & $0.37$ &        & $0.39$ & $0.51$ \\
L3 & $0.54$ & $0.39$ &        & $\mathbf{0.67}$ \\
L4 & $0.56$ & $0.51$ & $\mathbf{0.67}$ &       \\
\bottomrule
\end{tabular}
\caption{\textbf{Pairwise Jaccard overlap of the four cascade dimensions.} Computed over failed homogeneous rollouts. L3 and L4 are the most entangled pair (bold).}
\label{tab:cascade_jaccard}
\end{table}

\section{Case Studies: One Failure per Cascade Stage}
\label{app:cases}

To complement the quantitative cascade analysis of Section~\ref{sec:mechanism} and Appendix~\ref{app:cascade}, we walk through one illustrative homogeneous-mode failure per cascade stage.

\subsection{L1 Grounding: \texttt{streetmap}, \texttt{claude-sonnet-4-5}}

\paragraph{Setup.} A $7 \times 7$ street map. Agent~1 sees one-way arrows and \texttt{.} cells; Agent~2 sees walls (\texttt{\#}), arrows, and \texttt{.} cells. Both have \texttt{@} at $(0,0)$ and \texttt{*} at $(6,6)$. Agent~1's view contains no \texttt{\#} cells, so any claim about walls must be sourced from Agent~2.

\paragraph{What happened.} In turn~2 Agent~2 produces a ``combined map'' before Agent~1 has shared anything mappable to coordinates. The walls Agent~2 lists at $(1,3), (4,2), (5,2), (6,3), (6,4), (0,6)$ do not exist at those positions in Agent~2's own grid; they correspond to Agent~2's true walls with row and column axes swapped (true walls at $(3,1), (2,4), (2,5), (3,6), (4,6), (6,0)$). Agent~2 fabricates a global state from its private grid under the wrong axis convention. Agent~1 corrects coordinates in turn~3 but never challenges the wall list, so the fabricated walls propagate into the agreed route.

\begin{dialoguebox}
[turn 2, Agent 2]
Analysis:
- Start: (0,0), Goal: (6,6)
- I see walls at (1,1), (1,3), (3,3),
  (4,2), (5,2), (6,3), (6,4), (0,6)
- One-way streets at (2,1), (0,2), (2,3),
  (4,4), (0,5), (1,6)
\end{dialoguebox}

The grounding chain (``I see walls at $X$'') is broken at the moment of utterance.

\paragraph{Why this is L1.} A query was made (``what does the right half look like?''), an integration block was produced, and the receiving agent verified the route step by step. The chain is sound; the upstream input was fictional.

\subsection{L2 Querying: \texttt{logic}, \texttt{gpt-4o-mini}}

\paragraph{Setup.} A 4-position Einstein-style puzzle. Each agent holds six private clues. Each agent must elicit clues from the partner; emitting one's own clues serially does not establish that the partner has shared theirs.

\paragraph{What happened.} Across seven turns the agents take turns \emph{announcing} a single new clue per message, make a deduction from the union of clues stated so far, and never ask the partner for a specific fact. No utterance of the form ``please state whether the Spanish person drinks water'' appears in the transcript.

\begin{dialoguebox}
[turn 1, Agent 1]
The red house must be in position 1, 2, or 3 ...
I will share my next clue: the tea drinker is NOT
the blue house.

[turn 2, Agent 2]
Since the Spanish person is in position 1, the red
house must be in position 2 or 3 ...
I'll add: the green house is to the LEFT of yellow.

[turn 6, Agent 2]
Yes, that is correct! ... ACTI!
\end{dialoguebox}

The final answer matches $1$ of $12$ attribute cells; score $0.083$.

\paragraph{Why this is L2.} Every clue \emph{that was stated} came from one of the agents' private clue lists, so claims that were made are grounded. The problem is the absence of a query mechanism: the second-message query mandate from the L2 intervention clause would have forced an explicit ask about the missing positional constraints that the agents instead inferred from thin air.

\subsection{L3 Integration: \texttt{maze}, \texttt{DeepSeek-V4-Pro}}

\paragraph{Setup.} A $7 \times 7$ split-view maze. Both views contain enough partial information that an integration step is required before the first move.

\paragraph{What happened.} The transcript is four messages long. Agent~1 surveys their own grid, proposes a single first move, asks an open question, and Agent~2 immediately terminates.

\begin{dialoguebox}
[turn 0, Agent 1]
Looking at my map: ... my current position is at
(0,0) marked with @.
I'll propose my first move: move right to (0,1).
What do you see on your map?

[turn 1, Agent 2]
Move right to (0,1).

[turn 2, Agent 1]
I agree. Let's both move right to (0,1).
... What do you see after moving?

[turn 3, Agent 2]
ACTI!
\end{dialoguebox}

There is no combined-state block, no enumeration of cells either agent has resolved, and no plan beyond move~1. The decisive claim is two coordinates long, $(0,0)$ and $(0,1)$; score $0.083$.

\paragraph{Why this is L3.} L4 requires the receiver to re-derive before accepting, and one cannot re-derive nothing. The integration step where the two agents would have written down their merged map state never occurs. This case also illustrates the observation in Appendix~\ref{app:cascade} that L3 has $0\%$ sole-fire rate: when integration is absent, the proposal is empty enough that L4 has no content either, so L3 and L4 co-fire.

\subsection{L4 Re-derivation: \texttt{blueprintmaze}, \texttt{DeepSeek-V4-Pro}}

\paragraph{Setup.} A $7 \times 7$ split-blueprint maze. Agent~1 sees the left columns and the start \texttt{@}; Agent~2 sees the right columns and the goal \texttt{*}. Start at $(3,3)$, goal at $(2,6)$.

\paragraph{What happened.} The exchange runs cleanly through L1 (no fabricated cells), L2 (Agent~1 asks ``where is the goal? what does the right half look like?''), and L3 (Agent~2 prints an explicit combined map). Agent~2 then proposes a route and emits \texttt{ACTI!} in the same message, without waiting for Agent~1 to re-walk the path.

\begin{dialoguebox}
[turn 1, Agent 2]
Combined map (0-indexed rows 0-6, columns 0-6):
Row 0: . . # # . . .
Row 1: . # . . . . .
Row 2: . # # # . . *
Row 3: . . @ . # . .   <- start (3,3)
Row 4: # # . . # . .
Row 5: . . . . . . .
Row 6: . . # . . # .

Proposed route: (3,3), (3,4), (3,5), (3,6), (2,6)
Check walls: (3,4) is '.', (3,5) is '.', (3,6) is
'.', (2,6) is '*'. ACTI!
\end{dialoguebox}

The integration block is inconsistent with itself: the combined map lists $(3,4)$ as \texttt{\#}, but the route check immediately below says $(3,4)$ is \texttt{.}. A re-derivation pass by Agent~1 would have caught the contradiction in the very block Agent~2 wrote.

\paragraph{Why this is L4.} Integration \emph{was} produced (the explicit combined-map block). The breakdown is at the final accept step: the receiver was given no chance to recompute, and the proposer self-terminated. The L4 clause requiring the other agent to show recomputation work before either emits \texttt{ACTI!} is exactly the structural fix.

\subsection{Cross-case summary}

\begin{itemize}
  \item \textbf{L1} failures are sourcing errors: the surface form of the claim is well-formed English but the content is invented (here by axis transposition of the speaker's own grid).
  \item \textbf{L2} failures are silences: each utterance is well-formed; what is missing is a question.
  \item \textbf{L3} failures are truncations: the dialogue ends or commits before a merged-state utterance ever appears.
  \item \textbf{L4} failures are premature \texttt{ACTI!}: the proposer self-confirms instead of the receiver re-deriving.
\end{itemize}

The four cases come from three task categories (Spatial, CSP, Spatial, Spatial), consistent with the per-category bottleneck pattern of Section~\ref{sec:intervention}: Spatial path tasks expose L1, L3, and L4 most cleanly, while CSP logic puzzles expose the absence of L2 queries.

\subsection{Mapping the cascade to unstructured tasks}
\label{app:cases:transfer}

The four stages are defined over utterances rather than over grids or graphs, so each has a
reading in settings where the answer is not algorithmically checkable. In a paired
code-modification dialogue, L1 grounding is the requirement that a claim about a symbol, a call
site, or a test result name the file or the run it came from, and its failure is the familiar
invented API. L2 querying is asking the partner what a function they own actually returns
instead of inferring it from its name. L3 integration is stating the merged picture, which
change lands in which module and what each implies for the other, before either party starts
editing. L4 re-derivation is the reviewer re-executing the reasoning rather than approving on
the strength of a confident description. The same reading carries to issue triage: sourcing a
symptom to a log line, asking the reporter for the missing reproduction step, combining both
accounts before assigning a cause, and re-checking that cause against the trace.

Two things do not transfer automatically. Our judge operationalises each stage against a
deterministic ground truth, which unstructured tasks lack, so a transferred judge would need
either human labels or a task-specific verifier; and the relative ordering of the four
bottlenecks is a property of the task family rather than a constant, so which stage dominates in
a production workload is an empirical question rather than a prediction of this paper. What we
expect to transfer is the decomposition itself, and the finding that its stages are behavioural
and therefore addressable by protocol design rather than only by a stronger model.

\section{Full Task Specifications}
\label{app:tasks}

The suite spans $32$ tasks across three structural families. Table~\ref{tab:tasks} summarises every task with its answer type and one-line grader rule. We then deep-dive on a single representative task per family, chosen for pedagogical clarity rather than for being typical: \textbf{maze} (Spatial, the canonical task of \citealp{davidson2025collaboration}), \textbf{relaquery} (Relational), and \textbf{schedule} (CSP). The remaining $29$ tasks follow the same structure with task-specific instance generators, splitters, and graders; the full per-task definitions are released as part of the code accompanying this paper.

\begin{table*}[t]
\centering
\small
\resizebox{\textwidth}{!}{%
\begin{tabular}{l l l p{8cm}}
\toprule
Family & Task & Answer & Grader \\
\midrule
\multirow{11}{*}{Spatial}
& \texttt{maze}          & coordinate sequence       & valid-prefix length / optimal length, brute-forced over $16$ schemes \\
& \texttt{blueprintmaze} & coordinate sequence       & valid-prefix length / optimal length \\
& \texttt{conflictpath}  & two coordinate sequences  & joint valid moves / optimal joint plan \\
& \texttt{cube}          & 3D coordinate sequence    & valid-prefix length / optimal length \\
& \texttt{floodfill}     & set of reachable cells    & IoU against the true reachable set \\
& \texttt{geomaze}       & coordinate sequence       & valid-prefix length / optimal length \\
& \texttt{numberlink}    & multi-path coordinate sequences & fraction of endpoint pairs correctly connected \\
& \texttt{rulemaze}      & coordinate sequence       & valid-prefix length / optimal length under colour-rule \\
& \texttt{streetmap}     & coordinate sequence       & valid-prefix length / optimal length under one-way rules \\
& \texttt{terrain}       & coordinate sequence       & optimal cost / achieved cost \\
& \texttt{timedpath}     & timed coordinate sequence & time-feasible valid-prefix length / optimal length \\
\midrule
\multirow{11}{*}{Relational}
& \texttt{aliasedtaxonomy} & node-pair alignment      & per-pair correctness \\
& \texttt{ancestorpath}    & list of ancestor paths   & per-query exact match \\
& \texttt{dagorder}        & topological order        & $1 -$ normalised Kendall tau \\
& \texttt{eulerpath}       & edge sequence            & fraction of edges covered without repeat \\
& \texttt{graph}           & coordinate sequence on graph & optimal-path correctness \\
& \texttt{graphattr}       & per-query answers        & per-query exact match \\
& \texttt{graphpath}       & edge sequence            & fraction of correct edges \\
& \texttt{noderanking}     & ordered top-$3$ nodes    & ordered set IoU \\
& \texttt{relaquery}       & per-query answers        & per-query exact match \\
& \texttt{treesum}         & numerical sum            & exact value \\
& \texttt{trianglecount}   & set of triangles         & set IoU \\
\midrule
\multirow{10}{*}{CSP}
& \texttt{attrcombo}     & per-query answers      & per-query exact match \\
& \texttt{binpack}       & item-to-bin assignment & fraction of items respecting capacity \\
& \texttt{cipherdecode}  & decoded string         & character accuracy \\
& \texttt{crossword}     & filled grid            & cell accuracy \\
& \texttt{eventlog}      & per-query answers      & per-query exact match \\
& \texttt{logic}         & full assignment        & cell accuracy \\
& \texttt{meetingmatrix} & schedule assignment    & fraction of constraints satisfied \\
& \texttt{order}         & total order            & $1 -$ normalised Kendall tau \\
& \texttt{schedule}      & event-to-slot assignment & fraction of events placed in a feasible slot \\
& \texttt{seating}       & circular order         & $1 -$ normalised circular distance \\
\bottomrule
\end{tabular}%
}
\caption{\textbf{The 32-task suite by structural family.} Each row gives the answer type produced by the agents and the one-line grader rule, normalised to $[0, 1]$. Deep-dive specifications for the three representative tasks (\texttt{maze}, \texttt{relaquery}, \texttt{schedule}) appear below; the full task-specific definitions for all $32$ are released with the code.}
\label{tab:tasks}
\end{table*}

\paragraph{Provenance of the task genres.} The suite is procedurally generated, but the
task genres are not invented. Table~\ref{tab:provenance} names, for each task, the published
game or canonical problem it instantiates. Five are games people play (\texttt{maze};
\texttt{numberlink}, the Nikoli puzzle popularised as Flow Free; \texttt{crossword};
\texttt{logic}, the Zebra/Einstein grid puzzle; \texttt{seating}), four more are maze variants,
and the remaining twenty-three are standard problems with deployed applications, among them
multi-agent path finding as used in warehouse robotics, bin packing, cross-timezone meeting
scheduling, topological sort, and shortest path. What is synthetic is the procedural instance
generation and the union-necessary split of an instance into two private views, and the split is
the experimental manipulation itself: it is what converts a familiar single-player problem into
a coordination probe.

\begin{table}[h]
\centering
\small
\begin{tabular}{l p{4.6cm}}
\toprule
Task & Instantiates \\
\midrule
\multicolumn{2}{l}{\emph{Spatial}} \\
\texttt{maze}          & maze; the split-view maze of \citet{davidson2025collaboration} \\
\texttt{numberlink}    & Numberlink (Nikoli) / Flow Free \\
\texttt{conflictpath}  & multi-agent path finding (warehouse robotics, air traffic) \\
\texttt{streetmap}     & routing on one-way street networks \\
\texttt{terrain}       & least-cost path over weighted terrain (GIS routing) \\
\texttt{timedpath}     & time-dependent routing (transit, delivery planning) \\
\texttt{floodfill}     & flood fill / connected components \\
\texttt{blueprintmaze} & floor-plan navigation \\
\texttt{cube}          & 3-D voxel navigation \\
\texttt{geomaze}       & maze under compass directions \\
\texttt{rulemaze}      & maze under colour-rule movement constraints \\
\midrule
\multicolumn{2}{l}{\emph{Relational}} \\
\texttt{eulerpath}     & Euler path; route inspection \\
\texttt{dagorder}      & topological sort (build systems, prerequisites) \\
\texttt{noderanking}   & centrality / PageRank-style ranking \\
\texttt{trianglecount} & triangle counting (clustering coefficient) \\
\texttt{graph}, \texttt{graphpath} & shortest path on a weighted graph \\
\texttt{ancestorpath}, \texttt{relaquery} & genealogy and kinship queries \\
\texttt{treesum}       & subtree aggregation (org-chart roll-ups, \texttt{du}) \\
\texttt{graphattr}     & knowledge-graph attribute queries \\
\texttt{aliasedtaxonomy} & ontology alignment, entity resolution \\
\midrule
\multicolumn{2}{l}{\emph{CSP}} \\
\texttt{crossword}     & crossword puzzle \\
\texttt{logic}         & Zebra / Einstein logic-grid puzzle \\
\texttt{seating}       & seating-arrangement puzzle \\
\texttt{binpack}       & bin packing (logistics, VM placement) \\
\texttt{schedule}, \texttt{meetingmatrix} & cross-timezone meeting scheduling \\
\texttt{cipherdecode}  & substitution-cipher cryptanalysis \\
\texttt{order}         & rank aggregation from pairwise comparisons \\
\texttt{eventlog}      & event-log / timeline queries (process mining) \\
\texttt{attrcombo}     & combinatorial attribute queries \\
\bottomrule
\end{tabular}
\caption{\textbf{Real-world provenance of each task genre.} The procedural generator and the
union-necessary two-view split are ours; the underlying problem in every row is a published game
or a standard problem with deployed applications.}
\label{tab:provenance}
\end{table}

\paragraph{Spatial representative: \texttt{maze}.}
The answer is a sequence $a = (a_{0}, a_{1}, \ldots, a_{k})$ of $(\text{row}, \text{col})$ coordinates from start $a_{0}$ to goal $a_{k}$ on an $n \times n$ grid. The full instance is a grid of cells labelled wall (\texttt{\#}), path (\texttt{.}), start (\texttt{@}), goal (\texttt{*}). The partition independently masks roughly half of each agent's cells (cells become \texttt{?}) such that the union recovers the full grid. The grader extracts the proposed sequence from the dialogue, brute-forces over all $16$ origin / orientation / axis-order combinations following \citet{davidson2025collaboration}, and reports the prefix of moves that are valid in the original grid, normalised by the optimal-path length.

\paragraph{Relational representative: \texttt{relaquery}.}
The answer is a list of $q$ relational queries answered as triples $(s, \text{rel}, o)$ over a directed family DAG with \texttt{parent\_of} edges. The full instance is the union of two edge-disjoint subgraphs, one per agent (Pattern B split): every parent edge belongs to exactly one agent's view, so queries that traverse cross-partition paths require explicit information exchange. The grader scores per-query exact match against the ground-truth ancestry computed on the merged DAG, averaged over the $q$ queries.

\paragraph{CSP representative: \texttt{schedule}.}
The answer is an assignment of $m$ events to time slots respecting joint availability and time-zone constraints. The full instance is a set of participants with per-zone availability windows together with $m$ events, each requiring a subset of participants. The partition splits availability information: each agent sees the windows for one half of the participants, so feasibility for any cross-half event requires the agents to combine their views. The grader counts the number of events placed in a slot where all required attendees are available in their local zone, divided by $m$.

\section{Prompts}
\label{app:prompts}

\subsection{Protocol Prompts}
\label{app:prompts:protocol}

\paragraph{Collaboration system prompt.} Each agent in a paired rollout receives the following system prompt, which establishes the relay protocol with the partner.

\begin{promptbox}{System Message}
The user will act as an intermediary between you and another agent.\\
The user will directly forward your messages to the other agent and vice versa.\\
The user will not see or modify the messages, but will relay them as is.\\
Messages coming from the other agent will be prefixed with ``\texttt{[other agent]:}''.\\
Messages coming from the user will be prefixed with ``\texttt{[user]:}''.\\
Do not add any additional prefixes or suffixes to your own messages.
\end{promptbox}

\paragraph{Critic prompt.} In the solo modes, after producing a candidate answer, the same model receives the following critic prompt before grading.

\begin{promptbox}{User Message}
Carefully review the final solution you have provided above.\\
Ensure that it is complete, valid, and consistent with every rule and constraint of the task.\\
If needed, make adjustments to ensure your answer meets all requirements.\\
After you have finished reviewing, please submit your final solution.
\end{promptbox}

\paragraph{Grader prompt.} The grader is invoked once per rollout with the full dialogue. It extracts the final proposed answer from the natural-language exchange and returns a structured YAML object containing both the parsed answer and, where the surface representation can vary (coordinate origin, axis order, naming convention), the inferred answer schema.

\begin{promptbox}{User Message}
You will be given a dialogue produced by one or two agents attempting to solve the task above.\\
Your job is to extract the final proposed answer from the dialogue, infer the answer schema where applicable, and return the result as a YAML object.\\
Do not include any commentary or analysis outside the YAML.

\medskip
\# Answer schema\\
\texttt{<task-specific YAML schema>}

\medskip
\# Dialogue\\
\texttt{<dialogue transcript>}
\end{promptbox}

\subsection{Intervention Clauses}
\label{app:prompts:intervention}

Each clause below is appended to the collaboration system prompt during the corresponding intervention condition of Section~\ref{sec:intervention}. The \emph{all-four} condition appends all four; each \emph{no~L$k$} condition appends three.

\paragraph{L1 grounding clause.}
\begin{promptbox}{Appended to system prompt}
For every quantitative claim you make (a value, count, weight, position, edge, or mapping), cite the source: either ``from my view, \ldots'' or ``as [other agent] said in turn $N$, \ldots''. Claims without a cited source must be marked as \texttt{(assumed)} and may not appear in the final answer.
\end{promptbox}

\paragraph{L2 query mandate.}
\begin{promptbox}{Appended to system prompt}
On your second message, before making any claim about positions or connectivity that your partner might know better, ask one specific question of the form ``Please state $X$.'' If your partner cannot answer, mark $X$ as unknown. Do not act on assumed values for $X$.
\end{promptbox}

\paragraph{L3 integration block.}
\begin{promptbox}{Appended to system prompt}
Before any agent issues a final-answer proposal, that agent's preceding message must begin with an explicit integration block: ``Combined state from both views: [facts from my view + facts from partner's view + derived consequences].'' Proposals not preceded by an integration block in the same turn are invalid; the other agent should reject them and request integration.
\end{promptbox}

\paragraph{L4 re-derivation requirement.}
\begin{promptbox}{Appended to system prompt}
Before either agent issues \texttt{ACTI!}, the other agent must show recomputation work in their immediately preceding message: re-walk the path step by step, recompute the sum, re-check each constraint. ``I agree'' or ``confirmed'' without the recomputation is invalid; the proposer must reject it and ask for explicit re-derivation.
\end{promptbox}

\section{Additional Results}
\label{app:results}

\paragraph{The gap predictor transfers across the homogeneous / heterogeneous boundary.}
Because two instances of one model share correlated errors and identical surface conventions,
the M1 regression of Section~\ref{sec:predict} could in principle be fitting a same-model
artefact rather than a general conversation-shape-to-gap mapping. It is not. Refitting the same
specification, with the same structural features and the same grouped-by-task cross-validation,
within the heterogeneous cells alone gives $R^{2} = 0.35$ (Spearman $\rho = 0.58$, Pearson
$r = 0.64$) over $188$ cells. Fitting on the $350$ homogeneous cells and predicting the
heterogeneous cells the model never saw gives $R^{2} = 0.41$ ($\rho = 0.63$, $r = 0.66$), which
is close to the homogeneous out-of-fold figure. The mapping from conversation shape to gap is
therefore shared across the two regimes, so the diagnostic of Section~\ref{sec:predict} is not
confined to same-model pairs.

\paragraph{Separating integration reasoning from coordination.}
The fourth mode referred to in Section~\ref{sec:protocols:modes}, \emph{solo-distributed}, is
the second single-agent control. One agent receives both private views, each in the same partial
form an agent would see in a paired rollout, and must reconcile and merge them alone before
answering. It demands exactly the integration of two differently-grounded partial descriptions
that collaboration demands, but with no partner, no message passing, and no turn structure, so
it isolates the reasoning half of the task from its distributed execution. Writing
$\text{dist-drop} = s_{\text{solo-full}} - s_{\text{solo-dist}}$ and
$\text{comm-drop} = s_{\text{solo-dist}} - s_{\text{homo}}$ splits the tax additively into an
integration-as-reasoning term and a coordination term.

Table~\ref{tab:distdrop} reports the split with all eleven models pooled. Handing both partial
views to one agent costs between $0.026$ and $0.048$, so a single agent recovers $93$--$99\%$ of
its full-instance score from exactly the material a pair receives. The remaining drop to
homogeneous collaboration is three to ten times larger on every one of the eleven models, and
accounts for $80$--$86\%$ of the total gap. Integrating two views is cheap as reasoning and
expensive as a joint act, which is what licenses reading the four cascade stages of
Section~\ref{sec:mechanism} as behaviours rather than as capability limits. It is also
consistent with the intervention of Section~\ref{sec:intervention}: a prompt-level change that
touches no weights recovers a substantial fraction of the gap, which a genuine reasoning ceiling
could not.

\begin{table}[h]
\centering
\small
\begin{tabular}{lccc}
\toprule
& Spatial & Relational & CSP \\
\midrule
$s_{\text{solo-full}}$ & $0.748$ & $0.897$ & $0.867$ \\
$s_{\text{solo-dist}}$ & $0.699$ & $0.869$ & $0.841$ \\
$s_{\text{homo}}$      & $0.426$ & $0.697$ & $0.738$ \\
\midrule
dist-drop (reasoning)    & $+0.049$ & $+0.028$ & $+0.026$ \\
comm-drop (coordination) & $+0.273$ & $+0.172$ & $+0.103$ \\
\midrule
coordination share & $85\%$ & $86\%$ & $80\%$ \\
\bottomrule
\end{tabular}
\caption{\textbf{Decomposing the tax into integration reasoning and distributed execution.}
Mean scores over the eleven models at $50$ rollouts per cell. The coordination share is
$\text{comm-drop} / (\text{dist-drop} + \text{comm-drop})$.}
\label{tab:distdrop}
\end{table}

\paragraph{Per-task decomposition within each category.}
The per-task decomposition (Figure~\ref{fig:taskrank}) reproduces the category effect at finer resolution. The high end of the distribution is dominated by Spatial path tasks; the low end is dominated by CSP tasks whose answer is a short assignment over named entities. Two tasks, \texttt{numberlink} and \texttt{crossword}, fall short of the solo-tractability target and their ratio gaps are reported but treated as exploratory.

\begin{figure}[h]
  \centering
  \includegraphics[width=\columnwidth]{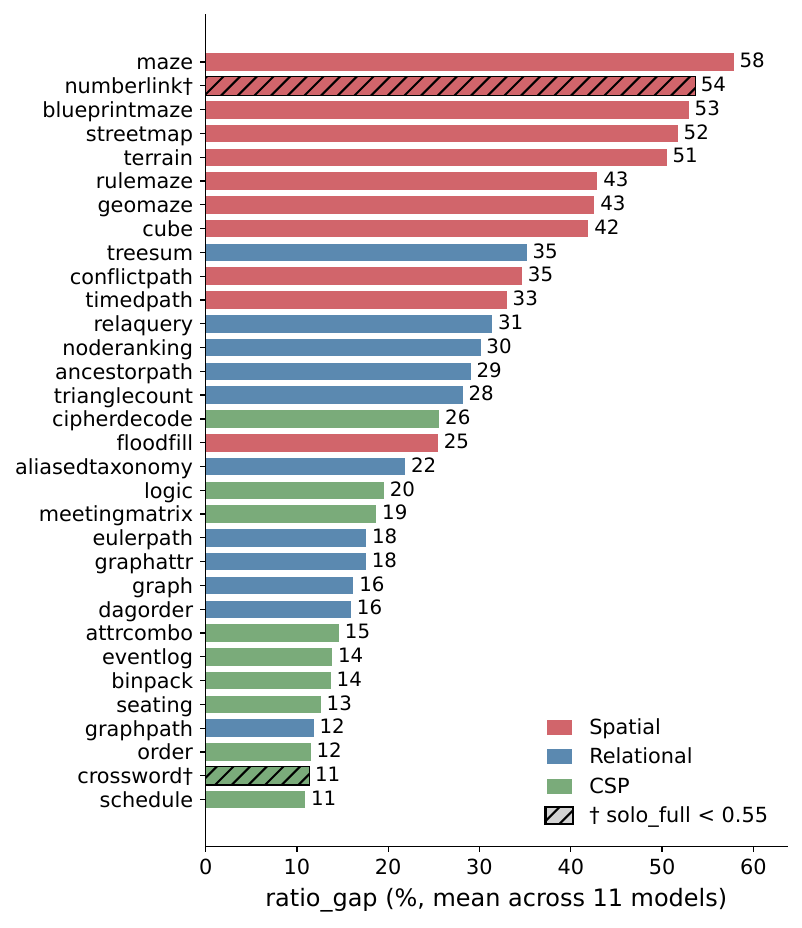}
  \caption{\textbf{Per-task ratio gap, by category.} Mean ratio gap on the homogeneous mode, by task, averaged across the eleven models. Colours indicate task category. The $\dagger$ marker flags the two tasks (\texttt{numberlink}, \texttt{crossword}) whose solo-full success rate falls short of the solo-tractability target; their ratio gaps are reported but treated as exploratory.}
  \label{fig:taskrank}
\end{figure}

\begin{figure}[h]
  \centering
  \includegraphics[width=\columnwidth]{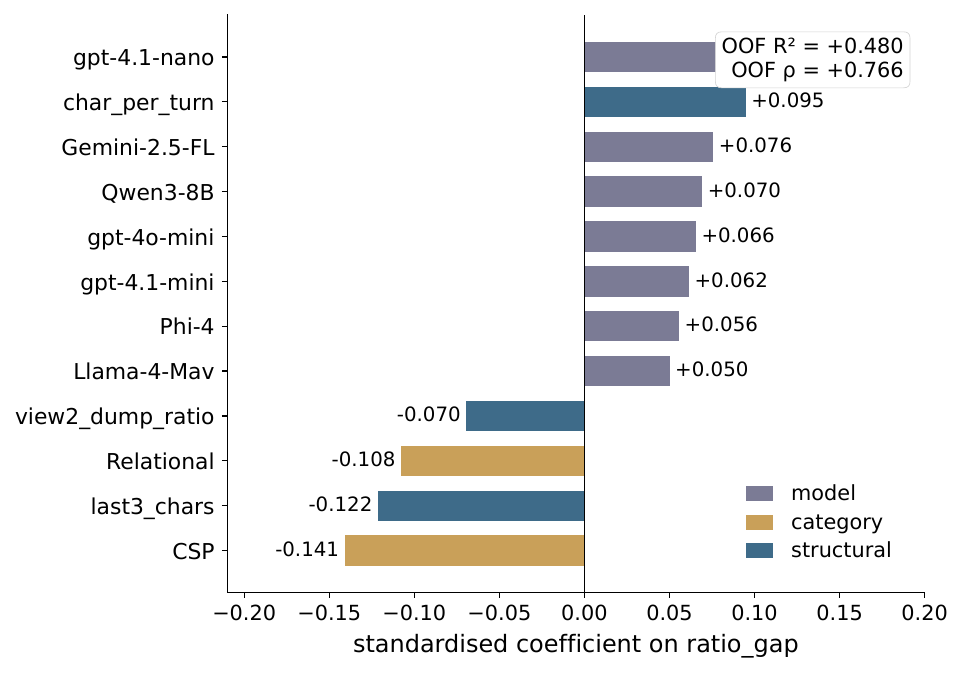}
  \caption{\textbf{Top standardised coefficients of the M1 ridge.} M1 fits the per-cell ratio gap on conversation-structural features plus model and category dummies. Model dummies are coloured separately. The model achieves out-of-fold $R^{2} = 0.480$ and Spearman $\rho = 0.766$ under group $k$-fold by task.}
  \label{fig:forest}
\end{figure}
\begin{figure*}[t]
  \centering
  \includegraphics[width=\textwidth]{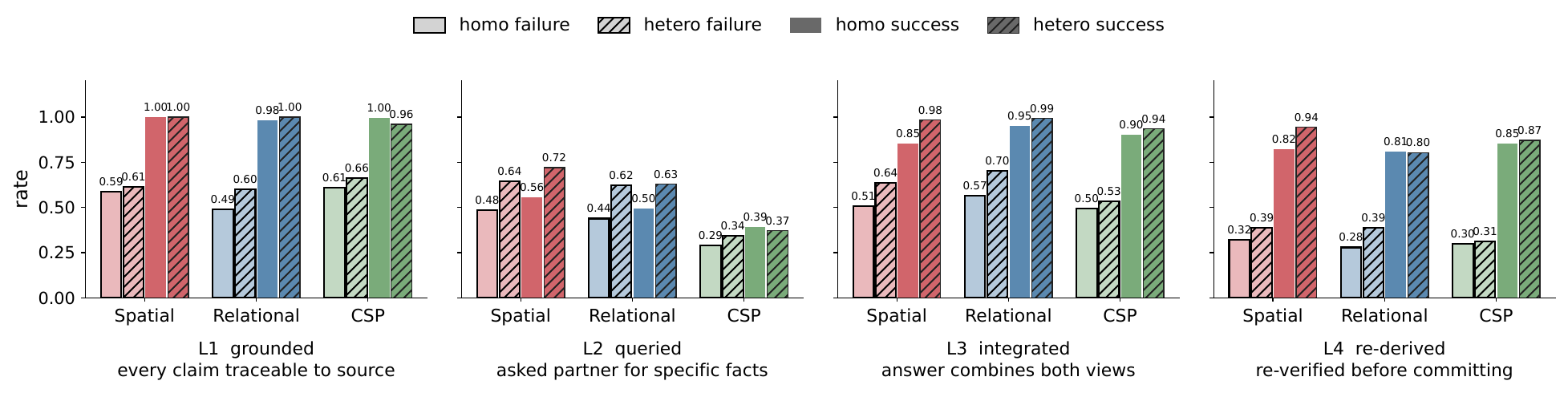}
  \caption{\textbf{Four-stage cascade in heterogeneous rollouts.} Layered onto the homogeneous bars of Figure~\ref{fig:mechanism}. Each panel reports the rate at which the corresponding cascade dimension fires, separated by category and outcome. The cascade signature continues to discriminate failure from success.}
  \label{fig:heteromech}
\end{figure*}
\paragraph{Per-cell heterogeneous heatmap.} Figure~\ref{fig:heteromat} reports the $3 \times 3$ heterogeneous-pair ratio-gap heatmap underlying the aggregated scatter of Figure~\ref{fig:additivenull}. The four off-diagonal cells cluster close to the strong-tier diagonal rather than averaging between strong and weak baselines.

\begin{figure}[h]
  \centering
  \includegraphics[width=\columnwidth]{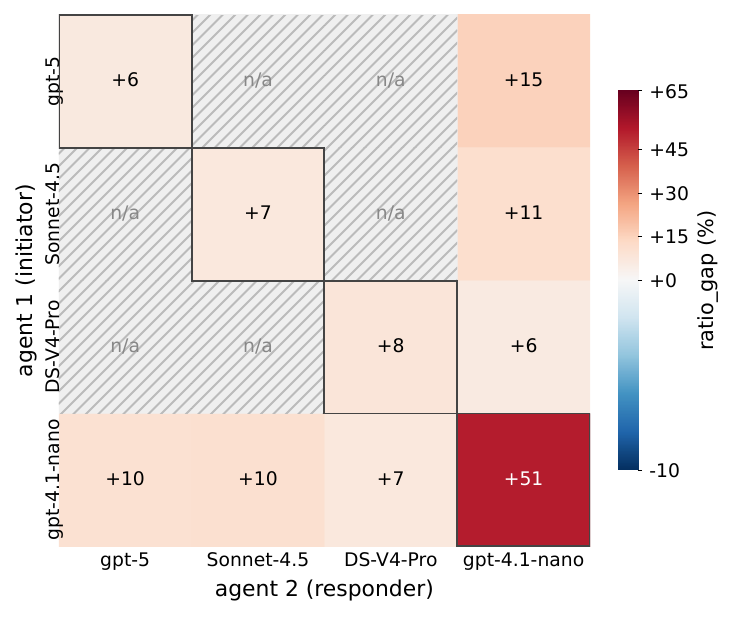}
  \caption{\textbf{Heterogeneous pair ratio gap heatmap.} Rows index the agent-1 (initiator) model, columns index the agent-2 (responder) model. Diagonal cells (boxed) reproduce the homogeneous baselines; off-diagonal cells are averaged across tasks. Hatched cells are strong-strong combinations, which were not run. The pair gap is pulled toward the stronger member.}
  \label{fig:heteromat}
\end{figure}

\paragraph{Cascade signatures transfer to heterogeneous rollouts.} We re-apply the four-stage cascade judge of Section~\ref{sec:mechanism} to the heterogeneous rollouts (Figure~\ref{fig:heteromech}). All four upstream stages continue to discriminate failure from success in the same direction as the homogeneous case: successful hetero pairs ground, query, integrate, and re-derive at high rates across categories, while failed hetero pairs hit each of these rates at substantially lower levels. The querying rate jumps most strongly: hetero pairs query more often than homogeneous pairs on Spatial and Relational, suggesting that pair asymmetry forces an explicit query that homogeneous pairs skip. The qualitative failure mode is the same as in the homogeneous case: failed coordination corresponds to a near-absence of the upstream cascade stages, irrespective of whether the two agents come from the same model or different ones.

\paragraph{Intervention side effects on dialogue length.}
A prompt that mandates citation, querying, integration, and re-derivation could raise scores
simply by making agents talk more. It does not. Averaged over the $31$ tasks with paired
baseline and intervention runs, the intervention \emph{shortens} the dialogue, from $18.2$ to
$15.9$ turns, while raising the mean score from $0.402$ to $0.439$ (Table~\ref{tab:sideeffects}).
Individual messages do get longer, from $7{,}854$ to $11{,}246$ characters per rollout, which is
the direct cost of the citation and integration clauses. Across tasks the correlation between
the change in turn count and the change in score is $-0.18$, so the lift is not bought with
extra turns; the largest drop is on Spatial, where baseline dialogues that used to spiral
(\texttt{blueprintmaze}, $65.9$ turns) terminate far earlier under the intervention ($21.4$).

\begin{table}[h]
\centering
\small
\begin{tabular}{lcccc}
\toprule
& \multicolumn{2}{c}{turns} & \multicolumn{2}{c}{chars} \\
\cmidrule(lr){2-3} \cmidrule(lr){4-5}
Category & base & all four & base & all four \\
\midrule
Spatial    & $36.0$ & $27.3$ & $11{,}527$ & $14{,}048$ \\
Relational & $9.2$  & $10.8$ & $5{,}224$  & $7{,}389$ \\
CSP        & $10.3$ & $10.1$ & $7{,}072$  & $12{,}689$ \\
\midrule
Overall    & $18.2$ & $15.9$ & $7{,}854$  & $11{,}246$ \\
\bottomrule
\end{tabular}
\caption{\textbf{Dialogue length under the all-four intervention.} Per-rollout means over the
$31$ tasks with paired runs. Turn count falls overall while the score rises, so the lift of
Table~\ref{tab:intervention} is not a turn-count artefact.}
\label{tab:sideeffects}
\end{table}

\paragraph{Critic ablation.}
\label{app:critic}
We ablate the solo critic pass on \texttt{gpt-4o-mini}, a lower-tier model on which the critic has the largest plausible room to lift solo performance. Across all $32$ tasks at $50$ rollouts per cell, removing the critic prompt drops the mean solo-full score from $0.575$ to $0.569$ (Table~\ref{tab:critic_ablation}), a $0.006$-point cost. The critic therefore contributes a much smaller fraction of the solo score than the gap to homogeneous collaboration spans on this model, so the tax is not an artefact of the critic step. Repeating the ablation on the two strongest models leaves the conclusion intact in both directions: the critic is worth $-0.001$ on \texttt{gpt-5}, which is already at ceiling, and $+0.046$ on \texttt{claude-sonnet-4-5}. The \texttt{claude-sonnet-4-5} gain sits entirely on Spatial path tasks, where the transcripts show the critic re-walking its own proposed route and deleting invalid detours: this is the L4 re-derivation mechanism of Section~\ref{sec:mechanism} applied by a single agent to its own answer, not added problem-solving capacity. In every case the critic moves the solo score \emph{up}, so removing it would if anything widen the reported tax.

\begin{table}[h]
\centering
\small
\begin{tabular}{lccc}
\toprule
Solo-full setting & \texttt{4o-mini} & \texttt{gpt-5} & \texttt{Sonnet} \\
\midrule
With critic pass    & $0.5753$ & $0.9940$ & $0.9450$ \\
Without critic pass & $0.5690$ & $0.9946$ & $0.8990$ \\
\midrule
$\Delta$ (critic contribution) & $+0.0063$ & $-0.0006$ & $+0.0460$ \\
\bottomrule
\end{tabular}
\caption{\textbf{Critic ablation across the capability range.} Mean solo-full score across the $32$ tasks at $50$ rollouts per cell, with and without the critic prompt of Appendix~\ref{app:prompts:protocol}. The strong-model columns are obtained by re-grading the pre-critic and post-critic message of the same saved solo-full transcripts, so no additional rollouts were run.}
\label{tab:critic_ablation}
\end{table}

\section{Termination and Turn-Cap Robustness}
\label{app:turncap}

The collaborative protocol caps a rollout at $50$ exchanges. This appendix reports how often
that cap binds, how the candidate answer is obtained when it does, and what the measured tax
looks like with every capped rollout removed. All numbers are recomputed from the saved
transcripts of the eleven canonical homogeneous runs; no additional rollouts were needed.

\paragraph{Extraction at the cap.} The grader is handed the entire transcript whether or not
\texttt{ACTI!} fired, and extracts the final proposed answer from it. There is no separate
mechanism for capped rollouts and they are not assigned zero by default: a capped rollout is
scored on the best answer its dialogue ever proposed.

\paragraph{The cap rarely binds.} Across $17{,}600$ homogeneous rollouts, $94.10\%$ terminate
naturally with \texttt{ACTI!}, $3.94\%$ reach the cap, and $1.95\%$ fail with an API error; the
cap rate among completed rollouts is $4.02\%$. It concentrates where the tax is largest without
becoming common there: Spatial $8.80\%$, Relational $1.52\%$, CSP $1.70\%$. Per model it tracks
capability inversely (Table~\ref{tab:capmodel}), from $14.09\%$ on
\texttt{gemini-2.5-flash-lite} down to $0.00\%$ on \texttt{gpt-5}. The four highest-capping
tasks are all Spatial paths: \texttt{maze} $20.8\%$, \texttt{streetmap} $17.4\%$,
\texttt{cube} $16.8\%$, \texttt{rulemaze} $16.0\%$.

\begin{table}[h]
\centering
\small
\begin{tabular}{lc}
\toprule
Model & cap rate \\
\midrule
\texttt{gemini-2.5-flash-lite} & $14.09\%$ \\
\texttt{gpt-4o-mini}           & $9.70\%$ \\
\texttt{gpt-4.1-nano}          & $9.19\%$ \\
\texttt{gpt-5-nano}            & $5.15\%$ \\
\texttt{gpt-4.1-mini}          & $2.00\%$ \\
\texttt{Qwen3-8B}              & $1.89\%$ \\
\texttt{Llama-4-Maverick}      & $1.19\%$ \\
\texttt{Phi-4}                 & $0.58\%$ \\
\texttt{DeepSeek-V4-Pro}       & $0.41\%$ \\
\texttt{claude-sonnet-4-5}     & $0.25\%$ \\
\texttt{gpt-5}                 & $0.00\%$ \\
\bottomrule
\end{tabular}
\caption{\textbf{Share of homogeneous rollouts reaching the $50$-exchange cap, by model.}
Computed over completed (non-error) rollouts.}
\label{tab:capmodel}
\end{table}

\paragraph{The cap does not inflate the tax.} Recomputing the ratio gap with every capped
rollout discarded moves the overall figure by $0.70$ percentage points, from $28.06\%$ to
$27.37\%$ (Table~\ref{tab:dropcap}). Both no-exception orderings survive: the category ordering
Spatial $\succ$ Relational $\succ$ CSP holds in the drop-cap numbers, and so does monotonicity
in capability. The largest shift for any single (model, category) cell is $4.71$ points
(\texttt{gemini-2.5-flash-lite}, Relational). Excluding the two exploratory tasks the overall
figure moves from $27.84\%$ to $27.10\%$.

\begin{table}[h]
\centering
\small
\begin{tabular}{lccc}
\toprule
Category & all rollouts & drop capped & $\Delta$ \\
\midrule
Spatial    & $44.26\%$ & $43.02\%$ & $-1.24$ \\
Relational & $23.58\%$ & $23.09\%$ & $-0.49$ \\
CSP        & $15.28\%$ & $14.96\%$ & $-0.33$ \\
\midrule
Overall    & $28.06\%$ & $27.37\%$ & $-0.70$ \\
\bottomrule
\end{tabular}
\caption{\textbf{Ratio gap with and without cap-terminated rollouts.} $\Delta$ in percentage
points. Both no-exception orderings are preserved.}
\label{tab:dropcap}
\end{table}

\paragraph{Capped rollouts are failures, not interrupted successes.} A capped rollout scores
$0.242$ on average against $0.644$ for a natural termination ($0.202$ vs $0.462$ on Spatial,
$0.278$ vs $0.710$ on Relational, $0.428$ vs $0.754$ on CSP). After a full $50$ exchanges these
dialogues had accumulated roughly a fifth of the available credit on Spatial, so they were not
near-wins cut short. The non-zero mean also confirms the extraction behaviour above. We read
the cap as a symptom of failed coordination rather than a cause of the measured tax.

\section{Hyperparameters}
\label{app:hparams}

\paragraph{Generation.} Agents are sampled at temperature $0.7$ with the model's default top-$p$ and no output-token cap. The collaborative dialogue is capped at $50$ exchanges per rollout, where one exchange is a turn from each agent. Each rollout starts from an empty context, so no state leaks across rollouts. We run $50$ independent rollouts per $(\text{task}, \text{mode}, \text{model or pair})$ cell, with consecutive integer seeds controlling both instance generation and the partition into views.

\paragraph{Grading.} The grader is a fixed model (\texttt{gpt-4o-mini} in our experiments), held constant across every cell of the design. In particular, the grader does \emph{not} change when the agents do, so heterogeneous-pair comparisons are not confounded by grader-side capability differences. The grader is queried at temperature $0.5$, lower than the agent temperature so that grading variance does not dominate downstream comparisons, and we report its continuous score in $[0,1]$.

\section{API Robustness}
\label{app:robustness}
All API calls flow through a unified \texttt{chat()} wrapper with a $120$-second per-call timeout. Transient network failures trigger up to $5$ retries with exponential backoff. We separately observed that path-task transcripts occasionally trigger Azure content filters when consecutive mask cells and the start or goal markers co-occur in the visible portion; for this class of failure we apply up to $6$ retries with longer backoff, and the path-task renderer emits only the visible cells of each view to reduce false triggers.

\end{document}